\documentclass[10pt,journal,compsoc]{IEEEtran}

\usepackage[nocompress]{cite}
\usepackage{graphicx}

\usepackage[linesnumbered,boxed,ruled,commentsnumbered]{algorithm2e}
\usepackage{subfigure}
\usepackage{url}
\usepackage{multirow}
\usepackage{arydshln}
\usepackage{amsfonts}
\usepackage{amsmath}
\usepackage{amsthm}
\newcommand{\tabincell}[2]{\begin{tabular}{@{}#1@{}}#2\end{tabular}}
\newtheorem{theorem}{Theorem}
\theoremstyle{definition}
\newtheorem{definition}{Definition}

\usepackage{xcolor}

\begin{document}

\title{Interpretable and Efficient Heterogeneous Graph Convolutional Network}

\author{Yaming~Yang,
        Ziyu Guan$^{*}$\thanks{* Corresponding author},
        Jianxin~Li,
        Wei~Zhao,
        Jiangtao~Cui,
        Quan~Wang%
\IEEEcompsocitemizethanks{
\IEEEcompsocthanksitem Y. Yang, Z. Guan, and W. Zhao are with the State Key Laboratory of Integrated Services Networks, School of Computer Science and Technology, Xidian University, Xi'an, China 710071.
E-mail: \{ymyang@stu., zyguan@, ywzhao@mail.\}xidian.edu.cn
\IEEEcompsocthanksitem J. Cui and Q. Wang are with the School of Computer Science and Technology, Xidian University, Xi'an, China 710071.
E-mail: \{cuijt@, qwang@\}xidian.edu.cn
\IEEEcompsocthanksitem J. Li is with the School of Information Technology, Deakin University, Australia.
E-mail: jianxin.li@deakin.edu.au
}%
}

\markboth{Journal of \LaTeX\ Class Files, ~Vol.~XX, No.~XX, June~2020}%
{Yang \MakeLowercase{\textit{et al.}}: Interpretable and Efficient Heterogeneous Graph Convolutional Network}

\IEEEtitleabstractindextext{%

\begin{abstract}
Graph Convolutional Network (GCN) has achieved extraordinary success in learning effective task-specific representations of nodes in graphs. However, regarding Heterogeneous Information Network (HIN), existing HIN-oriented GCN methods still suffer from two deficiencies: (1) they cannot flexibly explore all possible meta-paths and extract the most useful ones for a target object, which hinders both effectiveness and interpretability; (2) they often need to generate intermediate meta-path based dense graphs, which leads to high computational complexity. To address the above issues, we propose an interpretable and efficient Heterogeneous Graph Convolutional Network (ie-HGCN) to learn the representations of objects in HINs. It is designed as a hierarchical aggregation architecture, i.e., object-level aggregation first, followed by type-level aggregation. The novel architecture can automatically extract useful meta-paths for each object from all possible meta-paths (within a length limit), which brings good model interpretability. It can also reduce the computational cost by avoiding intermediate HIN transformation and neighborhood attention. We provide theoretical analysis about the proposed ie-HGCN in terms of evaluating the usefulness of all possible meta-paths, its connection to the spectral graph convolution on HINs, and its quasi-linear time complexity. Extensive experiments on three real network datasets demonstrate the superiority of ie-HGCN over the state-of-the-art methods.
\end{abstract}

}

\maketitle

\IEEEdisplaynontitleabstractindextext

\IEEEpeerreviewmaketitle

\IEEEraisesectionheading{\section{Introduction}\label{sec:intro}}

\IEEEPARstart{I}{n} the real world, a graph usually contains multiple types of objects (nodes) and links (edges), which is called a heterogeneous graph, or Heterogeneous Information Network (HIN) \cite{pathsim}. Properly learning representations of objects in an HIN can boost a variety of tasks such as object classification and link prediction \cite{yang-hin-survey}. There are two main challenges for this problem: 
(1) \textit{Effectiveness}. The rich structural proximities captured by meta-paths (paths with object types and link types as
nodes and edges) in an HIN are shown to be important and beneficial to many network mining tasks \cite{yang-hin-survey}. A desired HIN representation learning method should be able to exploit as many meta-paths as possible. Furthermore, not all the proximities described by meta-paths are equally important for a specific task. For example, suppose we have a task to classify the research ares of papers in DBLP. Paper $p_{1}$ is published in an interdisciplinary conference such as WWW, and is connected to term ``Web Search''. Paper $p_{2}$ is published in AAAI, and is connected to term ``Graph Algorithm''. Obviously, the connected term of $p_{1}$ is more helpful for classifying $p_{1}$ as ``information retrieval'', while the conference where $p_{2}$ is published is more helpful to classify $p_{2}$ as ``artificial intelligence''. Therefore, an effective representation learning method should be able to discover and leverage ``personalized'' useful meta-paths for an object.
(2) \textit{Efficiency}. Real world HINs often contain a large number of objects and links. Hence, computational efficiency is also an important requirement for representation learning methods, especially considering that modern deep learning methods are usually trained on GPUs with limited memory. 



Existing HIN representation learning methods mainly fall into two categories: graph embedding methods, and Graph Convolutional Network (GCN) based methods.
The HIN embedding methods learn object representations in a non-parametric way and by preserving some specific structural properties.
Among them, some methods \cite{heer,pte,eoe} only preserve first-order proximity conveyed by relations. Although the other methods \cite{metapath2vec,shne,esim,herec,mcrec,hine,hin2vec} preserve high-order structural proximities conveyed by meta-paths, they either require users to specify meta-paths \cite{metapath2vec,shne,esim,herec,mcrec} or cannot learn the importance of meta-paths for a task \cite{hine,hin2vec}.
To summarize, (1) with unsupervised structure-preserving training, the learned embeddings may not lead to optimal performance for a specific task; (2) none of these methods can automatically explore useful meta-paths from all possible meta-paths for specific tasks.  

Recently, GCN has been successfully applied to many graph analytical tasks such as node classification. Different from graph embedding, GCN encodes structural properties by convolution and uses task-specific objectives for training. Several recent works try to extend GCN to HINs. However, they still fail to fully and efficiently exploit the structural properties of HINs. Table \ref{tab:related-work} summarizes the key deficiencies of existing HIN GCN methods:
(1) Some of them \cite{han,hahe,deephgnn,magnn} require the user to specify several useful meta-paths for a specific task, which is difficult for users without professional knowledge.
(2) Many of them \cite{han,hahe,deephgnn,graphinception,activehne,hetgnn} cannot exploit all possible meta-paths, risking potential loss of important structural information. They only exploit a subset of all possible meta-paths, such as user-specified symmetric meta-paths \cite{han,hahe,deephgnn,magnn}, fixed-length meta-paths \cite{activehne}, or meta-paths that start from and end with the same object type \cite{graphinception}. HetGNN \cite{hetgnn} samples neighbors for a target object by random walk and aggregates them by Bi-LSTM. Some important structural information may be lost in this process.
(3) Some methods \cite{graphinception,activehne,r-gcn,decagon} do not distinguish the importance of meta-paths, failing to consider that not all meta-paths are useful for a specific task.
(4) Many of them \cite{han,hahe,deephgnn,magnn,graphinception,gtn} need to compute commuting matrices \cite{pathsim} by iterative multiplication of adjacency matrices, which has at least square time complexity to the number of involved objects. The resulting commuting matrices are very dense, and the longer the meta-paths, the denser the commuting matrices, which also increases the time complexity of the final graph convolution using these commuting matrices. Thus, those methods cannot scale well to large-scale HINs.

\begin{table}
\centering
\tabcolsep=0.14cm
\caption{Summary of Related Methods. (1) NU - not require user prior knowledge? (2) AMP - exploit all possible meta-paths? (3) UMP - automatically discover useful meta-paths? (4) LS - linear or quasi-linear scalability?}
\renewcommand{\arraystretch}{1.3}
\begin{tabular}{|c|c|c|c|c|c|c|c|}
\hline
Property & \tabincell{c}{\cite{r-gcn,decagon} \\ \cite{hetsann,hgt}}  & \tabincell{c}{\cite{han,hahe} \\ \cite{deephgnn,magnn}} & \cite{activehne,hetgnn} & \cite{graphinception} & \cite{gtn} & ie-HGCN \\
\hline
NU & $\surd$ & $\times$ & $\surd$ & $\surd$ & $\surd$ & $\surd$ \\
\hline
AMP	& $\surd$ & $\times$ & $\times$ & $\times$ & $\surd$ & $\surd$ \\
\hline
UMP	& $\times$ & $\surd$ & $\times$ & $\times$ & $\surd$ & $\surd$ \\
\hline
LS & $\surd$ & $\times$ & $\surd$ & $\times$ & $\times$ & $\surd$ \\
\hline
\end{tabular}
\label{tab:related-work}
\end{table}

Very recently, several HIN GCN methods \cite{gtn,hetsann,hgt} start to consider all possible meta-paths within a length limit. Among them, GTN \cite{gtn} first computes meta-path based graphs for all possible meta-paths, and then performs graph convolution. However, it has two disadvantages: 
(1) It only keeps a learnable importance weight for each relation. The weight is shared among all the related objects, which is not flexible enough to capture ``personalized'' important meta-paths for different objects. 
(2) It also needs to compute the commuting matrices (incorporating the relation weights) for each meta-path. Even by applying sparse-sparse matrix multiplication, it has at least square time complexity. Therefore it cannot well scale to large HINs (Section \ref{sec:scala}).  
HetSANN \cite{hetsann} and HGT \cite{hgt} directly aggregate the representations of heterogeneous neighbor objects by the multi-head attention mechanism \cite{gat}, and add a residual connection after each layer. However, (1) the interpretability of the model is hindered by the multi-head concatenation and residual connections, since they break the normalization property of probabilities and consequently it is difficult to assess the contribution of different parts; (2) in real-life power-law networks, objects could have very high degrees, which leads to calculation inefficiency of softmax \cite{word2vec-softmax} in attention and further affects scalability.

To fully and efficiently exploit structural properties of HINs, we propose an interpretable and efficient Heterogeneous Graph Convolutional Network (ie-HGCN), which directly takes an HIN as input and performs multiple layers of heterogeneous convolution on the HIN to learn task-specific object representations. Each layer of ie-HGCN has three key steps to obtain higher-level object representations:
(1) \textit{Projection}.
We define relation-specific projection matrices to project heterogeneous neighbor objects' hidden representations (input object features in the first layer) into a common semantic space corresponding to the target object type. We additionally define self-projection matrices (one for each object type) to project the representations of the target objects into the common semantic space as well. 
(2) \textit{Object-level Aggregation}. 
Given the adjacency matrix ${{\mathbf A}^{\Omega - \Gamma}}$ between ${\Omega}$-type target objects and their ${\Gamma}$-type neighbor objects, its row-normalized matrix ${\widehat {\mathbf A}^{\Omega - \Gamma}}$ is used to perform within-type aggregation among the neighbor objects of each target object.
We show the first two steps intrinsically define a heterogeneous spectral graph convolution operation on the bipartite graph described by ${{\mathbf A}^{\Omega - \Gamma}}$, with the projection matrices in the first step as convolution filters (Section \ref{sec:analys-spectral-conv}).
(3) \textit{Type-level Aggregation}.
We develop a type-level attention mechanism to learn the importance of different types of neighbors for a target object and perform type-level aggregation on the object-level aggregation results accordingly. 

Compared to existing HIN GCN methods, the proposed ie-HGCN has two salient features as follows:
\begin{enumerate}
    \item \textit{Interpretability}: The proposed object-level aggregation and type-level aggregation define a selection probability distribution in every object's neighborhood in each layer. By stacking multiple layers, the aggregation scheme facilitate adaptively learning the probabilistic importance score of each meta-path for each object, which assures ``personalized'' usefulness assessment of meta-paths and enhances the interpretability of the model. We formally prove that ie-HGCN can evaluate all possible meta-paths within a length limit (i.e., model depth) in Section \ref{sec:analys-meta-path}. 
    \item \textit{Efficiency}: ie-HGCN evaluates various meta-paths as the multi-layer iterative calculation proceeds. Hence, it avoids the computation of meta-path based graphs which is quite time-consuming. Moreover, in each layer ie-HGCN first uses row-normalized adjacency matrices (it is a reasonable choice and we will discuss it in Section~\ref{sec:obj-agg}) to aggregate a target object's neighbors of respective types as super ``type'' objects, and then uses type-level attention to aggregate them. This hierarchical aggregation architecture also makes our model efficient, because: (1) it avoids large-scale softmax calculation in the neighborhood of a target object; (2) an HIN typically only has a small number of object types, which leads to very efficient attention calculation. In Section \ref{sec:analys-time-complexity}, we analyze the time complexity of ie-HGCN to verify its quasi-linear scalability.
\end{enumerate}

The rest of this paper is organized as follows. In Section \ref{sec:relate}, we briefly introduce the related works. In Section \ref{sec:problem}, we formally define some concepts about HINs and the problem we study in this paper. In Section \ref{sec:model}, we present the proposed ie-HGCN model and analyze its theoretical properties. In Section \ref{sec:expe}, we conduct extensive experiments to show the superior performance of ie-HGCN against state-of-the-art methods on three benchmark datasets. Finally, we conclude this paper in Section \ref{sec:conclu}.

\section{Related Work}
\label{sec:relate}
\vspace{2mm} \noindent \textbf{HIN Embedding Methods:}
In recent years, a series of embedding methods are proposed to learn representations of objects in HINs.
EOE \cite{eoe}, PTE \cite{pte} and HEER \cite{heer} split an HIN into several bipartite graphs, and then use the LINE model \cite{line} to learn object representations by preserving the first-order or the second-order proximities. 
HIN2Vec \cite{hin2vec} learns representations of objects and meta-paths by predicting whether two objects have a specific relation. 
HINE \cite{hine} learns object representations by minimizing the distance between two distributions which respectively model the meta-path based proximity on the graph and the first-order proximity \cite{line} in the embedding space.
A lot of methods \cite{esim,metapath2vec,shne,herec,mcrec} first sample path instances guided by a set of user-specified meta-paths, and then learn object embeddings based on the resulting path instances.
Specifically, Esim \cite{esim} maximizes/minimizes the probability of each observed/noisy path instance. SHNE \cite{shne} and metapath2vec \cite{metapath2vec} learn object representations by their proposed heterogeneous skip-gram.
HERec \cite{herec} and MCRec \cite{mcrec} apply the HIN embedding idea to recommendation.
HERec performs type filtering to obtain homogeneous user/item sequences from the sampled path instances, and optimizes objective like skip-gram \cite{word2vec-softmax} to obtain user/item embeddings.
MCRec generates meta-path embeddings by using CNN to encode path instances, and then improves embeddings of users, items (initialized by matrix factorization) and meta-paths by co-attention with three-way interaction (user, meta-path, item).
However, these methods cannot learn task-specific embeddings. Although structural properties are exploited, none of them can automatically learn meta-path importance for all meta-paths within a length limit, not to mention task-specific importance.

\vspace{2mm} \noindent \textbf{GCNs for Homogeneous Graphs:}
Inspired by the great success of convolutional neural networks in computer vision, researchers try to generalize convolution on graphs \cite{gnn-survey-zhang}. 
Bruna \textit{et al.} \cite{bruna} first develop a graph convolution operation based on graph Laplacian in the spectral domain, inspired by the Fourier transformation in signal processing. 
Then, ChebNet \cite{chebnet} is proposed to improve its efficiency by using K-order Chebyshev polynomials. Kipf \textit{et al.} \cite{gcn} further introduce a first-order approximation of the K-order Chebyshev polynomials, to further build efficient deep models. 
Veli\v{c}kovi\'{c} \textit{et al.} propose GAT \cite{gat} to further learn different importance of nodes in a node's neighborhood by their proposed masked self-attention mechanism. 
Hamilton \textit{et al.} propose a general inductive framework GraphSAGE \cite{graphsage}. It learns to generate node embeddings by sampling neighbor nodes and aggregating their features by the proposed aggregator functions. 
All these methods are developed for homogeneous graphs. They cannot be directly applied to HINs because of heterogeneity. 

\vspace{2mm} \noindent \textbf{GCNs for Heterogeneous Graphs:} In recent years, a lot of HIN GCN methods are developed.
HAN \cite{han}, HAHE \cite{hahe}, DeepHGNN \cite{deephgnn} and GraphInception \cite{graphinception} transform an HIN into several homogeneous graphs based on user-specified symmetric meta-paths \cite{han,hahe,deephgnn}, or meta-paths that start from and end with the same object type \cite{graphinception}. Then, they apply GCN separately on each obtained homogeneous graph and aggregate the output representations by attention \cite{han,hahe,deephgnn} or concatenation \cite{graphinception}.
Given a set of meta-paths, MAGNN \cite{magnn} first performs intra-metapath aggregation by encoding all the object features along the path instances of a meta-path. Then it performs inter-metapath aggregation by attention to combine messages from multiple meta-paths. 
HetGNN \cite{hetgnn} first samples a fixed number of neighbors in the vicinity of an object via random walk with restart. Then it performs within-type aggregation of these neighbors with Bi-LSTM, and designs a type-level attention mechanism for type-level aggregation.
R-GCN \cite{r-gcn} and Decagon \cite{decagon} first perform within-type aggregation like GCN \cite{gcn}, and then sum the aggregation results of different types of neighbors.
ActiveHNE \cite{activehne} splits an HIN into several homogeneous/bipartite subgraphs and applies graph convolution on each subgraph separately. At the end of each convolutional layer, it concatenates the convolution results from those subgraphs as the output. 
GTN \cite{gtn} first computes meta-path based graphs for all possible meta-paths within a length limit by iterative matrix multiplication of the corresponding adjacency matrices weighted by learnable relation importance. Then it performs graph convolution on these resulting graphs.
HetSANN \cite{hetsann} and HGT \cite{hgt} extend GAT \cite{gat} to HINs. They directly calculate attention scores for all the neighbors of a target object and perform aggregation accordingly.
However, these methods either cannot discover useful meta-paths from all possible meta-paths \cite{han,hahe,deephgnn,graphinception,magnn,activehne,hetgnn,r-gcn,decagon,hetsann,hgt}, or have limited scalability \cite{han,hahe,deephgnn,graphinception,magnn,gtn}.

\section{Problem Formulation}
\label{sec:problem}
We first introduce some important concepts about HINs \cite{pathsim}, and then formally define the problem we study in this paper.

\begin{definition}
\label{def:hin}
\textbf{Heterogeneous Information Network (HIN).}
A heterogeneous information network is defined as $\mathcal{G} = (\mathcal{V}, \mathcal{E}, \phi, \psi)$, where $\mathcal{V}$ is the set of objects, $\mathcal{E}$ is the set of links. $\phi$ : $\mathcal{V} \to \mathcal{A}$ and $\psi$ : $\mathcal{E} \to \mathcal{R}$ are respectively object type mapping function and link type mapping function. $\mathcal{A}$ denotes the set of object types, and $\mathcal{R}$ denotes the set of relations, where $|\mathcal{A}|+|\mathcal{R}|>2$. Let $\mathcal{V}^\Omega$ denote the set of objects for type $\Omega \in \mathcal{A}$, and ${\mathcal{N}_\Omega} = \{ \Gamma|\Gamma,\Omega \in \mathcal{A}, {\langle {\Gamma,\Omega} \rangle} \in \mathcal{R}\}$ denote the set of neighbor object types of $\Omega$ that have relations from them to $\Omega$. $\Gamma \in {\mathcal{N}_\Omega}$ is a neighbor object type of $\Omega$. We abuse notation a bit to use $\Gamma$ also as the index of object type $\Gamma$ in ${\mathcal{N}_\Omega}$. The relation from $\Gamma$ to $\Omega$ is denoted as ${\langle {\Gamma,\Omega} \rangle}$ or $\Gamma \to \Omega$.
\end{definition}

\begin{definition}
\label{def:netschema}
\textbf{Network Schema.}
Given an HIN $\mathcal{G} = (\mathcal{V}, \mathcal{E}, \phi, \psi)$, $\phi$ : $\mathcal{V} \to \mathcal{A}$, $\psi$ : $\mathcal{E} \to \mathcal{R}$, the network schema is a directed graph defined over $\mathcal{A}$, with edges as relations from $\mathcal{R}$, denoted as $\mathcal{T} =(\mathcal{A}, \mathcal{R})$.
It is a meta template for $\mathcal{G}$.
\end{definition}

\begin{definition}
\label{def:metapath}
\textbf{Meta-path.}
A meta-path $\mathcal P$ is essentially a path defined on network schema $\mathcal{T}$.
It is denoted in the form of $A_1 \xrightarrow{R_1} A_2 \xrightarrow{R_2} \cdots \xrightarrow{R_l} A_{l+1}$ (abbreviated as $A_1A_2 \cdots A_{l+1}$), which describes a composite relation $R = R_1 \circ R_2 \circ \cdots \circ R_l$ between object types $A_1$ and $A_{l+1}$, where $\circ$ denotes the composition operator on relations. The subscript $l$ is the length of $\mathcal P$, i.e. the number of relations in $\mathcal P$. We say $\mathcal P$ is \textit{symmetric} if its corresponding composite relation $R$ is symmetric. A path instance of $\mathcal P$ is a concrete path in an HIN that instantiates $\mathcal P$.
\end{definition}

Figure \ref{fig:toy-dblp} shows a toy HIN of the DBLP bibliographic network (left) and its network schema (right). It contains 4 object types: ``Paper'' ($P$), ``Author'' ($A$), ``Conference'' ($C$) and ``Term'' ($T$), and 6 relations: ``Publishing'' and ``Published'' between $P$ and $C$, ``Writing'' and ``Written'' between $P$ and $A$, ``Containing'' and ``Contained'' between $P$ and $T$. For object type $P$, its set of neighbor object types is ${\mathcal{N}_P} = \{ C,A,T \}$. 
The meta-path $APA$ is symmetric, while the meta-path $CPA$ is asymmetric, and they both have length of 2. As shown in the figure, author $a_{2}$ has published paper $p_{2}$ in conference $c_{1}$, and thus we say $c_{1}p_{2}a_{2}$ is a path instance of $CPA$. A meta-path usually conveys a specific semantic meaning and different meta-paths have different importance for a specific task. For example, $APA$ means the co-author relationship between authors, while $CPA$ means authors publish papers in conferences. When predicting an author's affiliation, $APA$ is more helpful than $CPA$, since authors usually collaborate with colleagues in the same institution.

\begin{definition}
\label{def:hinlearn}
\textbf{HIN Representation Learning.}
Given an HIN $\mathcal{G}$ and a learning task $\mathcal{T}$, the problem is to learn $|\mathcal{A}|$ representation matrices for respective object types that are able to capture rich structural proximities and semantic information in $\mathcal{G}$ useful for $\mathcal{T}$. For each object type $\Omega \in \mathcal{A}$, the representation matrix is denoted as ${{\mathbf{X}}^\Omega} \in {{\mathbf{\mathbb{R}}}^{{|\mathcal{V}^\Omega|} \times d_\Omega}}$, where $d_\Omega \ll |\mathcal{V}^\Omega|$ is the representation dimensionality. For an object $v \in \mathcal{V}^\Omega$, its corresponding representation vector is the $v$-th row of ${{\mathbf{X}}^\Omega}$, which is a $d_\Omega$ dimensional vector.
\end{definition}

\begin{figure}
    \centering
    \includegraphics[width=0.8\columnwidth]{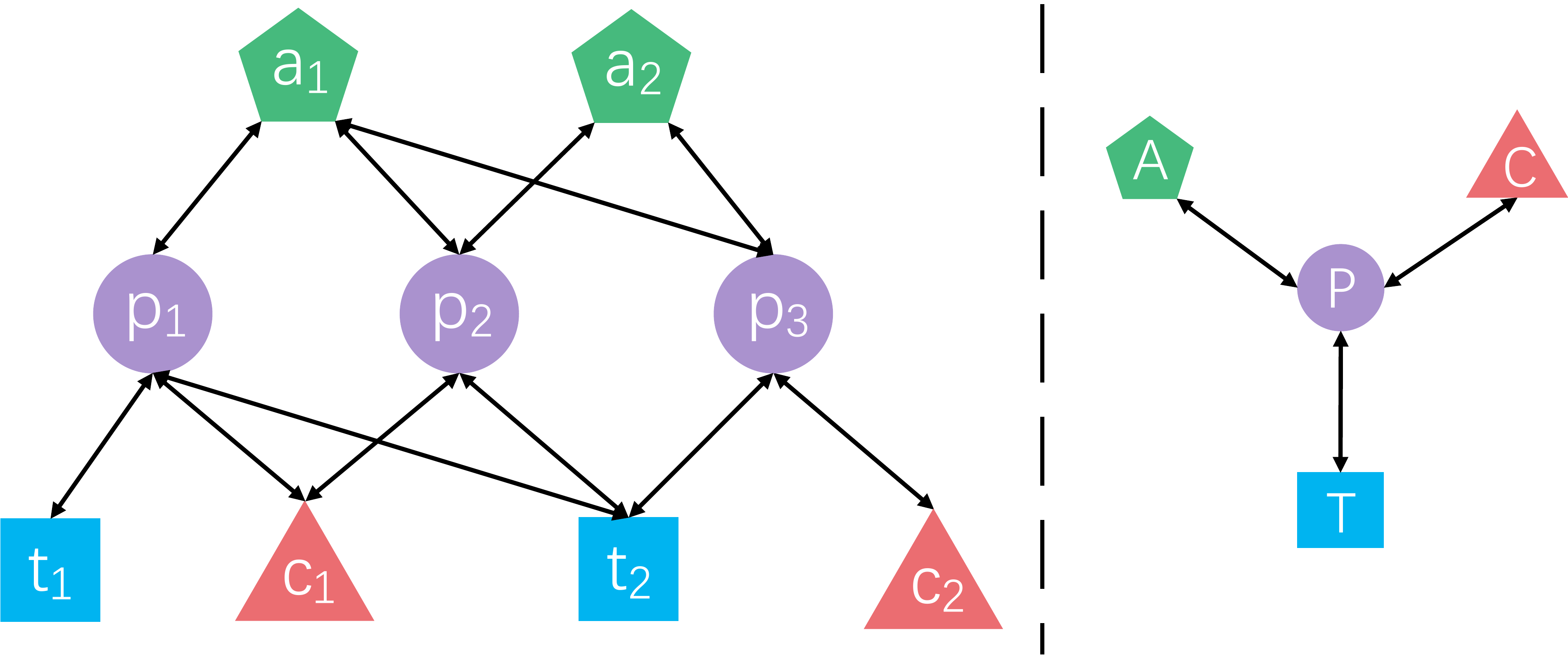}
    \caption{A toy HIN of DBLP (left) and its network schema (right).}
    \label{fig:toy-dblp}
\end{figure}

\section{Model}
\label{sec:model}

\begin{figure*}
\centering
\subfigure[An instance of ie-HGCN on DBLP]{\includegraphics[width=1.0\columnwidth]{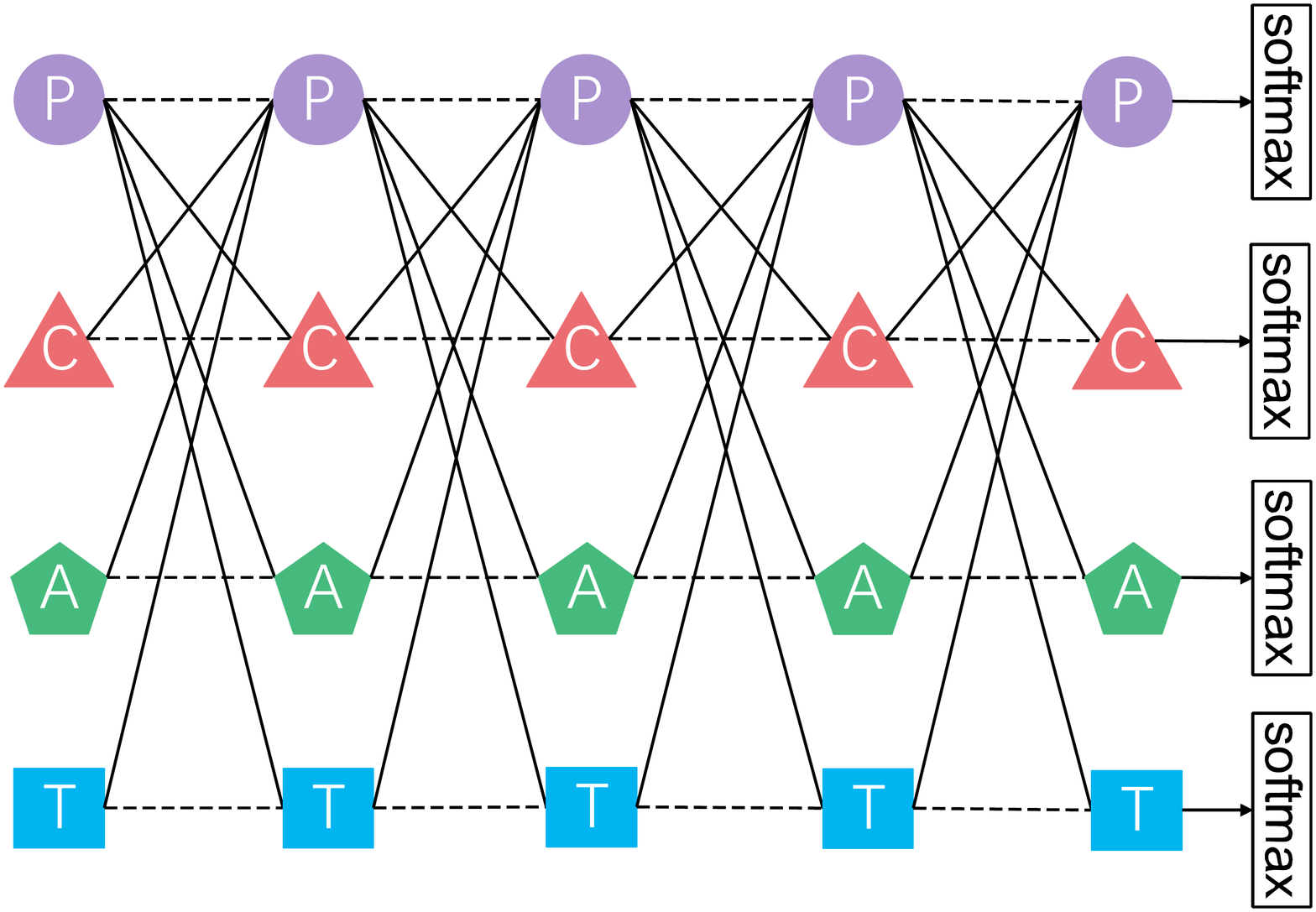}\label{fig:model-layer}}
\hspace{0.2cm}
\subfigure[The calculation flow in a $P$ block]{\includegraphics[width=1.0\columnwidth]{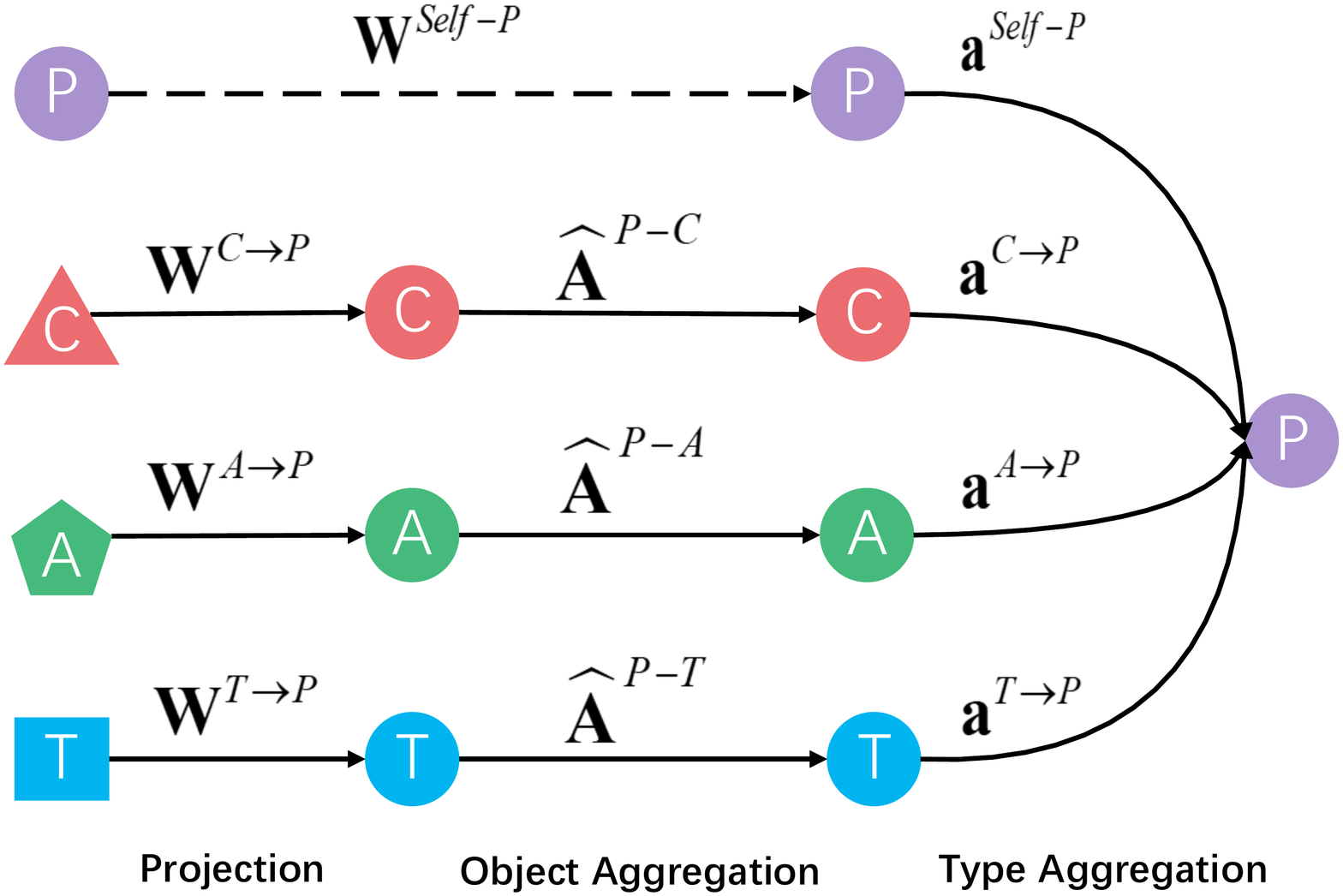}\label{fig:model-block}}
\caption{
The overall architecture of ie-HGCN on DBLP.
(a): An instance of ie-HGCN with 5 layers. The solid lines stand for the relation-specific projection, and the dashed lines stand for the dummy self-relation projection. In classification task, the softmax function can be applied to target object representations in the last layer to obtain prediction scores;
(b): The $P$ block in a layer. For $\Gamma \in {\mathcal{N}_P} = \{C, A, T\}$, ${{\mathbf W}^{\Gamma \to P}}$ projects the representations from $\Gamma$ semantic space into a new common ``Paper'' semantic space. ${{\mathbf W}^{Self-P}}$ projects the representations of paper objects from the original ``Paper'' semantic space into the new common ``Paper'' semantic space. 
We use the same shape to denote the projected object representations are located in the new common ``Paper'' semantic space. 
${\widehat {\mathbf A}^{P - \Gamma}}$ is used for object-level aggregation. Type-level attention is used for type-level aggregation.
}
\label{fig:model}
\end{figure*}

In this section, we present the ie-HGCN model. Figure \ref{fig:model-layer} shows the overall architecture of ie-HGCN on DBLP. Each layer consists of $|\mathcal{A}|$ blocks. In each block, three key calculation steps are performed. Figure \ref{fig:model-block} shows the calculation flow of the $P$ block in a layer. In the following, we elaborate on the three key calculation steps of the $\Omega \in \mathcal{A}$ block in a layer. The process is similar in other blocks. The main notations used in this paper are summarized in Table \ref{tab:notation}. We use bold uppercase/lowercase letters to denote matrices/vectors. For clarity, we omit layer indices of all the layer-specific notations.

\subsection{Projection}
\label{sec:proj}
For different types of objects, their features are located in different semantic spaces. To make these different types of object features comparable, in each block, we first project the representations of neighbor objects of different types into a new common semantic space. The input of the $\Omega$ block is a set of hidden representation matrices $\{ {{\mathbf H}^\Omega} \} \cup \{ {{\mathbf H}^\Gamma} | \Gamma \in {\mathcal{N}_\Omega} \} $ (input feature matrices in the first layer), obtained from the previous layer. ${{\mathbf H}^\Omega} \in {{\mathbb R}^{|{\mathcal{V}^\Omega}| \times {d_\Omega}}}$ and ${{\mathbf H}^\Gamma} \in {{\mathbb R}^{|{\mathcal{V}^\Gamma}| \times {d_\Gamma}}}$ are the representation matrices for $\mathcal{V}^\Omega$ and $\mathcal{V}^\Gamma$ respectively. For each neighbor object type $\Gamma \in {\mathcal{N}_\Omega}$, we define a \textit{relation-specific} projection matrix ${{\mathbf W}^{\Gamma \to \Omega}} \in {{\mathbb R}^{{d_\Gamma} \times d_\Omega^\prime }}$ corresponding to relation $\Gamma \to \Omega$. It projects ${{\mathbf H}^\Gamma}$ from the $\Gamma$ semantic space ${\mathbb R}^{d_\Gamma}$ into the new common semantic space ${\mathbb R}^{d_\Omega^\prime}$. Besides, to project ${{\mathbf H}^\Omega}$ from the lower feature space ${\mathbb R}^{d_\Omega}$ into the new common space ${\mathbb R}^{d_\Omega^\prime}$ as well, we additionally define a projection matrix ${{\mathbf W}^{Self - \Omega}} \in {{\mathbb R}^{{d_\Omega} \times d_\Omega^\prime }}$. Here ${{\mathbf W}^{Self - \Omega}}$ is simply a projection matrix, but not a relation-specific projection matrix. For convenience, we call ${Self - \Omega}$ as \textit{dummy self-relation}. When the real self-relation exists, i.e. ${\langle {\Omega,\Omega} \rangle} \in \mathcal{R}$, we use ${{\mathbf W}^{\Omega \to \Omega}}$ to denote the relation-specific projection matrix for ${\langle {\Omega,\Omega} \rangle}$. 
The projection is formulated as follows:
\begin{equation}
\label{eq:proj}
\begin{split}
{{\mathbf Y}^{Self - \Omega}} &= {{\mathbf H}^\Omega} \cdot {{\mathbf W}^{Self - \Omega}} \\
{{\mathbf Y}^{\Gamma \to \Omega}}      &= {{\mathbf H}^\Gamma} \cdot {{\mathbf W}^{\Gamma \to \Omega}}, \Gamma \in {\mathcal{N}_\Omega}
\end{split}
\end{equation}
where ${{\mathbf Y}^{Self - \Omega}} \in {{\mathbb R}^{|{\mathcal{V}^\Omega}| \times d_\Omega^\prime }}$ and ${{\mathbf Y}^{\Gamma \to \Omega}} \in {{\mathbb R}^{|{\mathcal{V}^\Gamma}| \times d_\Omega^\prime }}$ are projected hidden representations located in the new common space ${\mathbb R}^{d_\Omega^\prime}$.

For example, as illustrated in Figure \ref{fig:model-block}, ${{\mathbf W}^{C \to P}} \in {{\mathbb R}^{{d_C} \times d_P^\prime }}$ projects ${{\mathbf H}^C} \in {{\mathbb R}^{|{\mathcal{V}^C}| \times {d_C}}}$ from the ``Conference'' space ${\mathbb R}^{d_C}$ into a new common ``Paper'' space ${\mathbb R}^{d_P^\prime}$. ${{\mathbf H}^P} \in {{\mathbb R}^{|{\mathcal{V}^P}| \times {d_P}}}$ is originally located in the ``Paper'' space ${\mathbb R}^{d_P}$ for the previous layer. ${{\mathbf W}^{Self-P}} \in {{\mathbb R}^{{d_P} \times d_P^\prime }}$ projects it from the original space ${\mathbb R}^{d_P}$ into the new common space ${\mathbb R}^{d_P^\prime}$.

\begin{table}
    \tabcolsep=0.14cm
	\caption{Main Notations.}
	\label{tab:notation}
	\renewcommand{\arraystretch}{1.3}
	\begin{tabular}{ll}
		\hline
		Notations & Descriptions\\
		\hline
		${{\mathbf H}^\Omega}$ & Hidden representations of $\mathcal{V}^\Omega$ of previous layer\\
		${\mathbf{H}}^{\Omega \prime }$ & New representations of $\mathcal{V}^\Omega$ of current layer\\
		${{\mathbf W}^{Self - \Omega}}$ & Dummy self-relation projection matrix\\
		${{\mathbf W}^{\Gamma \to \Omega}}$ & Relation-specific projection matrix\\
		${{\mathbf Y}^{Self - \Omega}}$/${{\mathbf Y}^{\Gamma \to \Omega}}$ & Projected representations of $\mathcal{V}^\Omega$/$\mathcal{V}^\Gamma$\\
		${{\mathbf Z}^{Self - \Omega}}$ & i.e. ${{\mathbf Y}^{Self - \Omega}}$\\
		${{\mathbf Z}^{\Gamma \to \Omega}}$ & Aggregated representations from $\mathcal{V}^\Gamma$ to $\mathcal{V}^\Omega$\\
		${{\mathbf W}_q^\Omega}$/${{\mathbf W}_k^\Omega}$ & Attention query/key parameters\\
		${\mathbf w}_a^\Omega$ & Attention parameters\\
		${{\mathbf Q}^\Omega}$ & Mapped queries for $\mathcal{V}^\Omega$\\
		${{\mathbf K}^{Self - \Omega}}$/${{\mathbf K}^{\Gamma \to \Omega}}$ & Mapped keys for $\mathcal{V}^\Omega$/$\mathcal{V}^\Gamma$\\
		${{\mathbf e}^{Self - \Omega}}$/${{\mathbf e}^{\Gamma \to \Omega}}$ & Unnormalized attention coefficients for $\mathcal{V}^\Omega$/$\mathcal{V}^\Gamma$\\
		${{\mathbf a}^{Self - \Omega}}$/${{\mathbf a}^{\Gamma \to \Omega}}$ & Normalized attention coefficients for $\mathcal{V}^\Omega$/$\mathcal{V}^\Gamma$\\
		\hline
	\end{tabular}
\end{table}

\subsection{Object-level Aggregation}
\label{sec:obj-agg}
After projecting all the hidden representations of neighbor objects into a common semantic space, we then perform object-level aggregation. However, we cannot directly apply GCN \cite{gcn} to the aggregation, since the neighbors of an object are of different types in HINs, i.e., the heterogeneity of HINs. 
An adjacency matrix between two different types of objects may not even be a square matrix. In the following, let us take the example of aggregating hidden representations from ${\mathcal V}^{\Gamma}$ to ${\mathcal V}^{\Omega}$. Given the adjacency matrix ${{\mathbf A}^{\Omega - \Gamma}} \in {{\mathbb R}^{|{\mathcal{V}^\Omega}| \times |{\mathcal{V}^\Gamma}|}}$ between ${\mathcal V}^{\Omega}$ and ${\mathcal V}^{\Gamma}$, we first compute its \textit{row-normalized} matrix ${\widehat {\mathbf A}^{\Omega - \Gamma}} = {{({\mathbf D}^{\Omega - \Gamma}})}^{-1} \cdot {{\mathbf A}^{\Omega - \Gamma}}$, where ${\mathbf D}^{\Omega - \Gamma} = \text{diag}(\sum\nolimits_j {{\mathbf A}_{i,j}^{\Omega - \Gamma}}) \in {{\mathbb R}^{|{\mathcal{V}^\Omega}| \times |{\mathcal{V}^\Omega}|}}$ is the degree matrix. Then, we define the \textit{heterogeneous graph convolution} as follows:
\begin{equation}
\label{eq:hgcn}
\begin{split}
{{\mathbf Z}^{Self - \Omega}} &= {{\mathbf Y}^{Self - \Omega}} = {{\mathbf H}^\Omega} \cdot {{\mathbf W}^{Self - \Omega}} \\
     {{\mathbf Z}^{\Gamma \to \Omega}} &= {{\widehat {\mathbf A}}^{\Omega - \Gamma}} \cdot {{\mathbf Y}^{\Gamma \to \Omega}} \\
                 &= {{\widehat {\mathbf A}}^{\Omega - \Gamma}} \cdot {{\mathbf H}^\Gamma} \cdot {{\mathbf W}^{\Gamma \to \Omega}}, \Gamma \in {\mathcal{N}_\Omega}
\end{split}
\end{equation}
Each row of ${\widehat {\mathbf A}^{\Omega - \Gamma}}$ can serve as the normalized coefficients to compute a linear combination of the corresponding projected representations of ${\mathcal V}^{\Gamma}$. For symbolic consistency, we let ${{\mathbf Z}^{Self - \Omega}} = {{\mathbf Y}^{Self - \Omega}}$. Thus, we can obtain a set of convolved representations $\{ {{\mathbf{Z}}^{Self - \Omega}}, {{\mathbf{Z}}^{1 \to \Omega}}, \dots, {{\mathbf{Z}}^{\Gamma \to \Omega}}, \dots, {{\mathbf{Z}}^{|\mathcal{N}_\Omega| \to \Omega}} \}$, and each representation in the set contributes to $\mathcal{V}^\Omega$ from one aspect.
Take the $P$ block in Figure \ref{fig:model-block} as an example. We use ${\widehat {\mathbf A}^{P - C}}$, ${\widehat {\mathbf A}^{P - A}}$, ${\widehat {\mathbf A}^{P - T}}$ to respectively aggregate the projected representations of paper objects' neighbor conference objects, author objects and term objects. Thus, we obtain $\{ {{\mathbf{Z}}^{Self - P}}, {{\mathbf{Z}}^{C \to P}}, {{\mathbf{Z}}^{A \to P}}, {{\mathbf{Z}}^{T \to P}} \}$.

An alternative design choice is to employ an attention mechanism for object-level aggregation similar to that in \cite{gat}. However, we stick to using ${\widehat {\mathbf A}^{\Omega - \Gamma}}$ for object-level aggregation due to the following reasons: 
Firstly, (weighted) adjacency matrices could provide good enough weights for within-type aggregation (i.e., object-level aggregation). Intuitively, in a bibliography graph the papers an author has written equally contribute to his/her expertise; in a social activity graph the number of times a user visits an online shop naturally represent his/her preference degree towards that online shop.
Secondly, object-level attention is computationally inefficient, since in real-world complex networks, objects could have a large number of neighbors which results in calculation inefficiency of softmax in the attention mechanism \cite{word2vec-softmax}. By avoiding calculating softmax in large neighborhoods, the efficiency of ie-HGCN can be further enhanced. We will analyze the time complexity of ie-HGCN in Section \ref{sec:analys-time-complexity}.

Although Eq. (\ref{eq:hgcn}) is similar to the aggregation ideas in previous methods \cite{activehne,graphinception,r-gcn,decagon}, our design still has some novel aspects:
(1) Different from previous methods, we calculate the self-representation ${{\mathbf Z}^{Self - \Omega}}$, which, together with the attentive type-level aggregation introduced in the next subsection, enables ie-HGCN to evaluate the usefulness of all meta-paths within a length limit (model depth). We will prove this in Section~\ref{sec:analys-meta-path};
(2) Since ${{\widehat {\mathbf A}}^{\Omega - \Gamma}}$ is usually not a square matrix and consequently cannot be eigendecomposed to obtain Fourier basis, no previous work provides theoretical analysis to formally show Eq. (\ref{eq:hgcn}) is a proper convolution. In Section \ref{sec:analys-spectral-conv}, we will show that Eq. (\ref{eq:hgcn}) is intrinsically a spectral graph convolution on bipartite graphs;


\subsection{Type-level Aggregation}
\label{sec:type-agg}
To learn more comprehensive representations for $\mathcal{V}^\Omega$, we need to fuse representations from different types of neighbor objects. For a target object, the information from different types of neighbor objects could impact a specific task differently. Take paper objects in DBLP as an example. In the task of predicting a paper's quality, the representation of the conference where the paper is published could be more important.
To this end, we propose type-level attention to automatically learn the importance weights for different types of neighbor objects and aggregate the corresponding convolved representations from the previous step accordingly. 


The attention mechanism maps a set of queries and a set of key-value pairs to an output. In practice, we pack together queries, keys and values as rows into three matrices ${\mathbf Q}$, ${\mathbf K}$ and ${\mathbf V}$ respectively. Then it can be formulated as: $\text{softmax}(f({\mathbf Q},{\mathbf K})){\mathbf V}$, where $f$ is the attention function such as dot-product \cite{transformer} or neural network \cite{gat}. Here, the convolved representations from the previous step are values. We define a weight matrix ${{\mathbf W}_k^\Omega} \in {{\mathbb R}^{d_\Omega^\prime \times {d_a}}}$ to map them into keys, and define a weight matrix ${{\mathbf W}_q^\Omega} \in {{\mathbb R}^{d_\Omega^\prime  \times {d_a}}}$ to map ${{\mathbf Z}^{Self - \Omega}}$ into the query, where ${d_a}$ is the hidden layer dimensionality of the type-level attention. Formally,
\begin{equation}
\label{eq:query-keys}
\begin{split}
{{\mathbf Q}^\Omega}                &= {{\mathbf Z}^{Self - \Omega}} \cdot {{\mathbf W}_q^\Omega} \\
{{\mathbf K}^{Self - \Omega}}       &= {{\mathbf Z}^{Self - \Omega}} \cdot {{\mathbf W}_k^\Omega} \\
{{\mathbf K}^{\Gamma \to \Omega}}   &= {{\mathbf Z}^{\Gamma \to \Omega}} \cdot {{\mathbf W}_k^\Omega}, \quad \Gamma \in {\mathcal{N}_\Omega}
\end{split}
\end{equation}
It is intuitive to map $\{ {{\mathbf Z}^{Self - \Omega}} \} \cup \{ {{\mathbf Z}^{\Gamma \to \Omega}} | \Gamma \in {\mathcal{N}_\Omega} \}$ into keys, and map ${{\mathbf Z}^{Self - \Omega}}$ into the query, since we want to assess the importance of the convolved representations of neighbor types (including the self-representation) for each object in $\mathcal{V}^\Omega$. It is different from previous methods \cite{han,hahe,gtn}, where the query is a parameter vector. Note that mapping ${{\mathbf Z}^{Self - \Omega}}$ as the query is also the key to achieve personalized importance estimation for each $\Omega$ object. The attention function is implemented as follows:
\begin{equation}
\label{eq:typeatt}
\begin{split}
{{\mathbf e}^{Self - \Omega}} &= \text{ELU} \left( \left[ {{\mathbf K}^{Self - \Omega}}\| {{\mathbf Q}^\Omega} \right] \cdot {\mathbf w}_a^\Omega \right) \\
    {{\mathbf e}^{\Gamma \to \Omega}}    &= \text{ELU} \left( \left[ {{\mathbf K}^{\Gamma \to \Omega}}\| {{\mathbf Q}^\Omega} \right] \cdot {\mathbf w}_a^\Omega \right), \Gamma \in {\mathcal{N}_\Omega}
\end{split}
\end{equation}
where $\|$ denotes the row-wise concatenation operation, ${\mathbf w}_a^\Omega  \in {{\mathbb R}^{2{d_a} \times 1}}$ is the parameter vector, and ELU \cite{elu_nonlinear} is the activation function. The $i$-th element of ${{\mathbf e}^{Self}}$ and ${{\mathbf e}^{\Gamma}}$ respectively reflect the unnormalized importance of object $i$ itself and its $\Gamma$ neighbors when calculating its higher level representation. Then, the normalized attention coefficients are computed by applying the softmax function:
\begin{equation}
\label{eq:softmax}
\begin{split}
\left[ {{\mathbf a} ^ {Self - \Omega}}\| {{\mathbf a} ^ {1 \to \Omega}}\|...\|{{\mathbf a} ^ {\Gamma \to \Omega}}\|...\| {{\mathbf a} ^ {|{\mathcal{N}_\Omega}| \to \Omega}} \right] = \\
\text{softmax} \left( \left[ {{\mathbf e} ^ {Self - \Omega}}\| {{\mathbf e} ^ {1 \to \Omega}}\|...\|{{\mathbf e} ^ {\Gamma \to \Omega}}\|...\| {{\mathbf e} ^ {|{\mathcal{N}_\Omega}| \to \Omega}} \right] \right)
\end{split}
\end{equation}
where $\text{softmax}$ is applied to the operand row-wise.
The normalized attention coefficients are employed to compute the higher level representations of $\mathcal{V}^\Omega$ via a weighted combination of the corresponding values:
\begin{equation}
\label{eq:aggregate}
{\mathbf{H}}_{i,:}^{\Omega \prime } = \sigma \left( {{{\mathbf a}_i^{Self - \Omega}} \cdot {\mathbf Z}_{i,:}^{Self - \Omega} + \sum\limits_{\Gamma \in {\mathcal{N}_\Omega}} {{\mathbf a}_i^{\Gamma \to \Omega}} \cdot {{\mathbf Z}_{i,:}^{\Gamma \to \Omega}} } \right)
\end{equation}
where $\sigma$ is the nonlinearity, and the subscript $i$ ($i,:$) means the $i$-th element (row) of a vector (matrix), which corresponds to the $i$-th object in $\mathcal{V}^\Omega$. The new representations in ${\mathbf{H}}^{\Omega \prime }$ are in turn used as the input of the blocks in the next layer. The final representations of objects are output by the blocks in the last layer.

\subsection{Loss}
\label{sec:loss}
Once the final representations of objects are obtained from the last layer, they can be used for a variety of tasks such as classification, clustering, etc. The loss functions can be defined depending on specific tasks. For semi-supervised multi-class object classification, it can be defined as the sum (or weighted sum) of the cross-entropy over all the labeled objects for each object type:
\begin{equation}
\label{eq:loss}
{\mathcal L} = - \sum\limits_{\Omega \in {\mathcal A}} \sum\limits_{i \in {\mathcal I}^\Omega} \sum\limits_{j \in {\mathcal C}^\Omega} { {\mathbf L}_{i,j}^\Omega \cdot \text{ln} ({{\mathbf S}_{i,j}^\Omega}) }
\end{equation}
where ${\mathcal I}^\Omega$ is the set of indices of labeled objects in ${\mathcal V}^\Omega$, ${\mathcal C}^\Omega$ is the set of class indices for ${\mathcal V}^\Omega$, and ${\mathbf L}_{i,j}^\Omega$ and ${{\mathbf S}_{i,j}^\Omega}$ are respectively ground-truth label indicator and the predicted score of object $i \in {\mathcal I}^\Omega$ on class $j$. We can minimize the loss by back propagation. 
The overall training procedure of ie-HGCN is shown in Algorithm \ref{alg:hgcn}. Wherein, we index layers by square brackets.

\begin{algorithm}
\SetKwInOut{Input}{\textbf{Input}}\SetKwInOut{Output}{\textbf{Output}}
\Input{
	The HIN $\mathcal{G} = (\mathcal{V}, \mathcal{E}, \phi, \psi)$, $\phi$ : $\mathcal{V} \to \mathcal{A}$, $\psi$ : $\mathcal{E} \to \mathcal{R}$, \\
	The object feature matrices ${{\mathbf F}^\Omega}, \Omega \in \mathcal{A}$, \\
	The number of layers $N$. \\
}
\Output{
	The final representations ${{\mathbf H}^\Omega}[N], \Omega \in \mathcal{A}$.\\
}
\BlankLine
Initialize parameters, and let ${{\mathbf H}^\Omega}[1] = {{\mathbf F}^\Omega}, \Omega \in \mathcal{A}$ \;
\For{$n=2,...,N$}{
    \For{$\Omega \in \mathcal{A}$}{
        ${{\mathbf Z}^{Self - \Omega}}[n] = {{\mathbf H}^\Omega}[n-1] \cdot {{\mathbf W}^{Self - \Omega}}[n]$ \;
        \For{$\Gamma \in {\mathcal{N}_\Omega}$}{
            ${{\mathbf Z}^{\Gamma \to \Omega}}[n] = {{\widehat {\mathbf A}}^{\Omega - \Gamma}} \cdot {{\mathbf H}^\Gamma}[n-1] \cdot {{\mathbf W}^{\Gamma \to \Omega}}[n]$ \;
        }
        Compute normalized attention coefficients by ${{\mathbf W}_q^\Omega}[n]$, ${{\mathbf W}_k^\Omega}[n]$ and ${\mathbf w}_a^\Omega[n]$ according to Eq. (\ref{eq:query-keys}-\ref{eq:softmax})\;
        Compute ${\mathbf H}^\Omega[n]$ according to Eq. (\ref{eq:aggregate}) \;
    }
}
Compute loss and update parameters by gradient descent\;
\Return{${{\mathbf H}^\Omega}[N] , \Omega \in \mathcal{A}$. }
\caption{The pseudocode of ie-HGCN.}
\label{alg:hgcn}
\end{algorithm}

\subsection{Analysis}
\label{sec:analys}

\subsubsection{Automatically learning useful meta-paths}
\label{sec:analys-meta-path}
The most important highlight of ie-HGCN is that it evaluates task-specific importance of all possible meta-paths with length less than the model depth.
We formalize this property as a theorem as follows:
\begin{theorem}
\label{thm:autopath}
For an object type $\Omega \in \mathcal{A}$, let ${\mathcal P}_\Omega^{[0,n)}$ denote the set of all possible meta-paths of length greater than or equal to 0, less than $n$, and end with $\Omega$.
In the $n$-th layer, the output hidden representation ${{\mathbf H}^\Omega}[n]$ intrinsically evaluates task-specific importance of all the meta-paths in ${\mathcal P}_\Omega^{[0,n)}$.
\end{theorem}
\begin{proof}
\label{prof:autopath}
We prove the theorem by mathematical induction. For conciseness, ``${{\mathbf H}^\Omega}[n]$ evaluates ${\mathcal P}_\Omega^{[0,n)}$'' means ``${{\mathbf H}^\Omega}[n]$ evaluates the task-specific importance of all the meta-paths in ${\mathcal P}_\Omega^{[0,n)}$''.

\textbf{The base case}: When $n=1$, ${{\mathbf H}^\Omega}[1] = {{\mathbf F}^\Omega}$ is the input features of ${\mathcal V}^{\Omega}$. Obviously, the meta-path evaluated can be expressed as $\Omega$, which has a length of $0$ and ends with $\Omega$, i.e., ${\mathcal P}_\Omega^{[0,1)}$. In this case, the importance is defaulted to probability 1.

\textbf{The step case}: Assume that the theorem holds when $n=m-1 \geq 1$, i.e., ${{\mathbf H}^\Omega}[m-1]$ evaluates ${\mathcal P}_\Omega^{[0,m-1)}$. When $n=m$, ${{\mathbf H}^\Omega}[m]$ is an attention-weighted combination of ${{\mathbf Z}^{Self - \Omega}}[m]$ and ${{\mathbf Z}^{\Gamma \to \Omega}}[m], \forall \Gamma \in {\mathcal{N}_\Omega}$. According to Eq.~(\ref{eq:hgcn}), ${{\mathbf Z}^{Self - \Omega}}[m]$ is a linear projection of ${{\mathbf H}^\Omega}[m-1]$ which evaluates ${\mathcal P}_\Omega^{[0,m-1)}$ by assumption; ${{\mathbf Z}^{\Gamma \to \Omega}}[m]={{\widehat {\mathbf A}}^{\Omega - \Gamma}} \cdot {{\mathbf H}^\Gamma}[m-1] \cdot {{\mathbf W}^{\Gamma \to \Omega}}[m]$, where ${{\mathbf H}^\Gamma}[m-1]$ evaluates ${\mathcal P}_\Gamma^{[0,m-1)}$ by assumption. Due to the aggregation via ${\widehat {\mathbf A}}^{\Omega - \Gamma}$, the heterogeneous graph convolution corresponds to concatenating the relation $\Gamma \to \Omega$ at the end of each meta-path in ${\mathcal P}_\Gamma^{[0,m-1)}$. Since we perform aggregation for all $\Gamma \in {\mathcal{N}_\Omega}$, this results in ${\mathcal P}_\Omega^{[1,m)}$. By uniting ${\mathcal P}_\Omega^{[0,m-1)}$ from ${{\mathbf Z}^{Self - \Omega}}[m]$ and ${\mathcal P}_\Omega^{[1,m)}$ from ${{\mathbf Z}^{\Gamma \to \Omega}}[m]$, we can conclude ${{\mathbf H}^\Omega}[m]$ evaluates ${\mathcal P}_\Omega^{[0,m)}$. 

By mathematical induction, we can safely conclude that ${{\mathbf H}^\Omega}[n]$ evaluates ${\mathcal P}_\Omega^{[0,n)}$.
The importance of a meta-path is determined by the row-normalized adjacency matrices and the type-level attention coefficients across layers which are all normalized probability distributions. Note that the object-level and type-level aggregations in a block together define a selection probability distribution in a target node's neighborhood. For each path instance of a meta-path ending with the same target object, we can calculate a proper importance score by multiplying its corresponding selection probabilities across layers. The importance of this meta-path w.r.t. the target object can then be aggregated accordingly.

\end{proof}

The ie-HGCN can capture objects' personalized preference for different meta-paths because each object has its own attention coefficients. As stated in the proof of Theorem 1, the importance score of a meta-path w.r.t. a specific target object is obtained by summing the scores of all its path instances ending with that object. The score of a path instance is calculated by multiplying the selection probabilities between objects along the path, i.e., the attention coefficients multiplied by the link weights (from the corresponding row-normalized adjacency matrices for real relations, or 1 for dummy self-relations). Since path instances often share subpaths, we could efficiently aggregate subpath scores iteratively during the forward propagation of ie-HGCN, recording in each block the aggregation scores for all possible meta-paths up to that block. Finally, we need to merge equivalent meta-paths due to dummy self-relations.

To further improve efficiency, we could approximate the importance score of a meta-path in a global sense. Specifically, we calculate the mean attention distribution in each block. These mean distributions reflect the general trends of relation selection in these blocks. For example, for the task of predicting a paper's quality, the mean attention distribution in the top paper block would probably show a peak for the paper-conference relation. The global importance score of a meta-path could be calculated by treating these mean distributions as selection probability distributions over relations. Computation details are given in Section \ref{sec:attnstudy}.

\subsubsection{Connection to spectral graph convolution}
\label{sec:analys-spectral-conv}
We can also derive the heterogeneous graph convolution presented in Eq.~(\ref{eq:hgcn}) by connecting to the spectral domain of bipartite graphs (when the self-relation exists, the following derivation still holds by setting $\Gamma=\Omega$). For ${\mathcal{V}^\Omega}$ and ${\mathcal{V}^\Gamma}$, given their representation matrices ${\mathbf H}^\Omega$ and ${\mathbf H}^\Gamma$ and the adjacency matrices ${{\mathbf A}^{\Omega - \Gamma}}$ and ${{\mathbf A}^{\Gamma - \Omega}}$ between them, the challenge is that we cannot directly eigendecompose ${{\mathbf A}^{\Omega - \Gamma}}$ and ${{\mathbf A}^{\Gamma - \Omega}}$ as they may not be square matrices. Thus, we define the \textit{augmented adjacency matrix} $\widetilde{\mathbf A} \in {\mathbb R}^{(|{\mathcal{V}^\Omega}| + |{\mathcal{V}^\Gamma}|) \times (|{\mathcal{V}^\Omega}| + |{\mathcal{V}^\Gamma}|)}$ and the \textit{augmented representation matrix} $\widetilde{\mathbf H} \in {{\mathbb R}^{(|{\mathcal{V}^\Omega}| + |{\mathcal{V}^\Gamma}|) \times \text{max}({d_\Omega},{d_\Gamma})}}$ as follows:
$$
\widetilde{\mathbf A} =
\left[
  \begin{array}{c;{2pt/1pt}c}
                            \mathbf 0                         &   {{\mathbf A}^{\Omega - \Gamma}} \\
    \hdashline[3pt/2pt]
    \rule{0pt}{10pt}        {{\mathbf A}^{\Gamma - \Omega}}   &   \mathbf 0                       \\
  \end{array}
\right]
,\hspace{2mm}
\widetilde{\mathbf H} =
\left[
  \begin{array}{c}
    \rule{0pt}{10pt} {\mathbf H}^\Omega \\
    \hdashline[3pt/2pt]
    \rule{0pt}{10pt} {\mathbf H}^\Gamma \\
  \end{array}
\right]
$$
where $\mathbf 0$'s denote square zero matrices, and $\widetilde{\mathbf H}$ is properly padded by zeros since generally $d_\Omega \neq d_\Gamma$. Our convolution is related to \textit{random walk Laplacian} which is defined as: $\widetilde{\mathbf L}_{rw} = \mathbf I - {\widetilde{\mathbf D}}^{-1}  \widetilde{\mathbf A}$, where ${\widetilde{\mathbf D}} = \text{diag}(\sum\nolimits_j {{\widetilde{\mathbf A}}_{i,j}})$. We also have $\widetilde{\mathbf L}_{rw} = {\mathbf U} {\mathbf \Lambda} {\mathbf U}^{-1}$, where ${\mathbf U}$ and $\mathbf \Lambda$ are respectively $\widetilde{\mathbf L}_{rw}$'s eigenvectors and eigenvalues. ${\mathbf U}^{-1}$ and ${\mathbf U}$ define graph Fourier transform and inverse transform respectively. Then the bipartite graph convolution is defined as the multiplication of a parameterized filter $g_{\theta}=\text{diag}(\theta)$ ($\theta \in \mathbb{R}^{|{\mathcal{V}^\Omega}| + |{\mathcal{V}^\Gamma}|}$ in the Fourier domain) and a signal $\widetilde{\mathbf{h}}$ (a column of $\widetilde{\mathbf{H}}$) in the Fourier domain:
\begin{equation}\nonumber
\label{eq:spectural-conv}
\begin{split}
g_{\theta} \star \widetilde{\mathbf{h}} 
&= \mathbf{U} g_{\theta} \mathbf{U}^{-1} \widetilde{\mathbf{h}} \approx \mathbf{U} (\sum_{k=0}^{K} \theta_k T_k (\mathbf{\Lambda}) ) \mathbf{U}^{-1} \widetilde{\mathbf{h}} \\
&= \sum_{k=0}^{K} \theta_k T_k ({\mathbf U} {\mathbf \Lambda} {\mathbf U}^{-1}) \widetilde{\mathbf{h}} = \sum_{k=0}^{K} \theta_k T_k (\widetilde{\mathbf{L}}_{rw}) \widetilde{\mathbf{h}}
\end{split}
\end{equation}
where $g_{\theta}$ can be regarded as a function of ${\mathbf \Lambda}$ and is efficiently approximated by the truncated Chebyshev polynomials $T_k(x)$ \cite{chebnet}. $T_k(x) = 2xT_{k-1}(x) - T_{k-2}(x)$, $T_0(x)=1$, and $T_1(x)=x$. We note that in general, $\widetilde{\mathbf L}_{rw}$ has the same eigenvalues as symmetric normalized Laplacian $\widetilde{\mathbf L}_{sys}$ \cite{spectral_klopotek}, which lie in $[0, 2]$ \cite{spectral_chung}. For the purpose of numerical stability, we replace $\widetilde{\mathbf L}_{rw}$ with $\widetilde{\mathbf P} = \mathbf I - \widetilde{\mathbf L}_{rw}$ without affecting Fourier basis, so as to rescale the eigenvalues to [-1, 1] \cite{chebnet}. $\widetilde{\mathbf P}$ can be expressed as follows:
$$
\widetilde{\mathbf P} = {\widetilde{\mathbf D}}^{-1} \widetilde{\mathbf A} =
\left[
  \begin{array}{c;{2pt/1pt}c}
                            \mathbf 0                         &   {\widehat {\mathbf A}^{\Omega - \Gamma}} \\
    \hdashline[3pt/2pt]
    \rule{0pt}{10pt}        {\widehat {\mathbf A}^{\Gamma - \Omega}}   &   \mathbf 0                       \\
  \end{array}
\right]
$$
where ${\widehat {\mathbf A}^{\Omega - \Gamma}} \in {{\mathbb R}^{|{\mathcal{V}^\Omega}| \times |{\mathcal{V}^\Gamma}|}}$ is the row-normalized adjacency matrix between ${\mathcal V}^{\Omega}$ and ${\mathcal V}^{\Gamma}$. Now, the convolution operation can be expressed as: $\sum_{k=0}^{K} \theta_k T_k (\widetilde{\mathbf P}) \widetilde{\mathbf h}$, which is K-localized, since it can be easily verified that $(\widetilde{\mathbf P})^k$ denotes the transition probability of objects to their k-order neighborhood. Following GCN \cite{gcn}, we further let $K=1$ and stack multiple layers to recover a rich class of convolutional filter functions. Then we have $g_{\theta} \star \widetilde{\mathbf h} \approx \theta_0 \widetilde{\mathbf h} + \theta_1 \widetilde{\mathbf P} \widetilde{\mathbf h}$. Generalizing the filter to multiple ones, and the signal to multiple channels, the two terms can be expressed as follows:
$$
\widetilde{\mathbf H}{{\mathbf \Theta}_0} = \left[
  \begin{array}{c}
    \rule{0pt}{10pt} {\mathbf H}^\Omega {{\mathbf \Theta}_0}\\
    \hdashline[3pt/2pt]
    \rule{0pt}{10pt} {\mathbf H}^\Gamma {{\mathbf \Theta}_0}\\
  \end{array}
\right]
,\hspace{4mm}
\widetilde{\mathbf P} \widetilde{\mathbf H}  {{\mathbf \Theta}_1} =
\left[
  \begin{array}{c}
    \rule{0pt}{10pt} {\widehat {\mathbf A}^{\Omega - \Gamma}} {\mathbf H}^{\Gamma} {{\mathbf \Theta}_1}  \\
    \hdashline[3pt/2pt]
    \rule{0pt}{10pt} {\widehat {\mathbf A}^{\Gamma - \Omega}} {\mathbf H}^{\Omega} {{\mathbf \Theta}_1}  \\
  \end{array}
\right]
$$
The above two equations recover the calculation of $\mathbf{Z}^{Self - \Omega}$ ($\mathbf{Z}^{Self - \Gamma}$) and $\mathbf{Z}^{\Gamma \to \Omega}$ ($\mathbf{Z}^{\Omega \to \Gamma}$) in Eq.~(\ref{eq:hgcn}). They differ from Eq.~(\ref{eq:hgcn}) only by using the same parameters ${{\mathbf \Theta}_0}$ and ${{\mathbf \Theta}_1}$ for the two types $\Omega$ and $\Gamma$. In ie-HGCN, we use separate parameters, $\mathbf{W}^{Self - \Omega}$, ${{\mathbf W}^{\Gamma \to \Omega}}$ and $\mathbf{W}^{Self-\Gamma}$, ${{\mathbf W}^{\Omega \to \Gamma}}$ for $\Omega$ and $\Gamma$ respectively, to improve model flexibility. Another difference is that we aggregate $\mathbf{Z}^{Self - \Omega}$ ($\mathbf{Z}^{Self - \Gamma}$) and $\mathbf{Z}^{\Gamma \to \Omega}$'s ($\mathbf{Z}^{\Omega \to \Gamma}$'s) through the type-level attention rather than simply adding them.

\subsubsection{High computational efficiency}
\label{sec:analys-time-complexity}
Most previous methods \cite{han,hahe,deephgnn,graphinception,gtn} need to compute commuting matrices by iterative multiplication of adjacency matrices, which has at least square time complexity. Our ie-HGCN performs heterogeneous graph convolution on an HIN directly in each layer, which is more efficient.

The time complexity of ie-HGCN in the $\Omega$ block consists of two parts, corresponding to the heterogeneous graph convolution and the type-level attention respectively.
The first part is determined by Eq. (\ref{eq:hgcn}). 
The calculation of $\mathbf{Z}^{Self - \Omega}$ requires ${\mathcal O}( |{\mathcal V}^{\Omega}| \cdot {d_\Omega} \cdot {d_\Omega^\prime} )$.
By applying sparse-dense matrix multiplication, the calculation of $\mathbf{Z}^{\Gamma \to \Omega}$ requires ${\mathcal O}( |{\mathcal E}^{\Omega - \Gamma}| \cdot {d_\Gamma} \cdot {d_\Omega^\prime} )$, where $|{\mathcal E}^{\Omega - \Gamma}|$ is the number of links between ${\mathcal V}^{\Omega}$ and ${\mathcal V}^{\Gamma}$.
The second part is due to Eqs. (\ref{eq:query-keys} - \ref{eq:aggregate}):
Eq. (\ref{eq:query-keys}) requires ${\mathcal O}( (2 + |{\mathcal{N}_\Omega}|) \cdot (|{\mathcal V}^{\Omega}| \cdot {d_\Omega^\prime} \cdot {d_a}) )$;
Eq. (\ref{eq:typeatt}) requires ${\mathcal O}( (1 + |{\mathcal{N}_\Omega}|) \cdot (|{\mathcal V}^{\Omega}| \cdot 2{d_a}) )$;
Eq. (\ref{eq:softmax}) requires ${\mathcal O}( (1 + |{\mathcal{N}_\Omega}|) \cdot |{\mathcal V}^{\Omega}| )$;
Eq. (\ref{eq:aggregate}) requires ${\mathcal O}( (1 + |{\mathcal{N}_\Omega}|) \cdot (|{\mathcal V}^{\Omega}| \cdot {d_\Omega^\prime}) )$.

In practice, the dimensionality hyperparameters and the term $|{\mathcal{N}_\Omega}|$ are typically very small compared to the number of objects and links in the HIN. Thus, taking all the $N$ layers, all types of objects and all types of links into consideration, the overall time complexity is ${\mathcal O}( N \cdot (|{\mathcal V}| + |{\mathcal E}|) )$, which is linear to the total number of objects and links in an HIN.

\section{Experiments}
\label{sec:expe}
In this section, we conduct extensive experiments to show the performance of ie-HGCN. We use three widely used and publicly available real-world networks (IMDB, ACM, DBLP) to construct three HIN datasets, and compare the performance of ie-HGCN against ten baselines on these HINs.

Our method is implemented by PyTorch, and the source code is available at GitHub\footnote{\url{https://github.com/kepsail/ie-HGCN/}}. The source codes of the other baselines are provided by their authors, which are implemented by either PyTorch\footnote{\url{https://pytorch.org/}} or TensorFlow\footnote{\url{https://www.tensorflow.org/}}. All experiments are conducted on a server with 16 Intel Xeon E5-2620 CPUs, 1 Nvidia GeForce GTX 1080Ti with 12GB GPU memory, and 128GB main memory. Unless otherwise specified, all experiments are performed on GPU to accelerate computation.

\subsection{Datasets}
\label{sec:data}

\begin{table}
\centering
\caption{Dataset Statistics (the notation * marks real features).}
\renewcommand{\arraystretch}{1.3}
\begin{tabular}{|c|c|c|c|c|}
\hline
Dataset & Objects & Number & Features & Classes\\
\hline
\multirow{4}[0]{*}{DBLP}
&	A     & 4057  & 128 & 4	\\
\cline{2-5}
&	P     & 14328 & 128 & -	\\
\cline{2-5}
&	C     & 20    & 128 & -	\\
\cline{2-5}
&	T     & 8898  & 128 & -	\\
\hline
\multirow{3}[0]{*}{ACM}
&	P     & 4025  & 128$^*$ & 3 \\
\cline{2-5}
&	A     & 7167  & 128 & -	\\
\cline{2-5}
&	S     & 60    & 128	& - \\
\hline
\multirow{4}[0]{*}{IMDB}
&	M     & 3328  & 14$^*$ & 4 \\
\cline{2-5}
&	A     & 42553 & 128 & - \\
\cline{2-5}
&	U     & 2103  & 128 & - \\
\cline{2-5}
&	D     & 2016  & 128 & - \\
\hline
\end{tabular}
\label{tab:dataset}
\end{table}

The statistics of the used HINs are summarized in Table \ref{tab:dataset}. The notation * means the features are real. Otherwise, they are generated randomly. Note that, most existing methods only require features of the target objects, while ActiveHNE, GTN and our method need to input features of all types of objects. However, in the widely used HIN datasets, some types of objects have no available real features. For these objects, 
some existing methods input their one-hot ids as features, which results in a large number of parameters in the first layer and consequently, high space complexity and time complexity. Considering the general idea is to generate non-informative features for those objects without real features, in this paper,  we generate a 128-dimensional random vector for each of these objects from the Xavier uniform distribution \cite{xavier}. In this way, little information can be got from their features. For all the methods (except HAHE, which cannot make use of object input features), we input exactly the same object features as shown in Table \ref{tab:dataset}.

\textbullet\
\textbf{IMDB.}
We extract a subset from the IMDB dataset in HetRec 2011\footnote{\url{https://grouplens.org/datasets/hetrec-2011/}}, and construct an HIN which contains 4 object types: Movie ($M$), Actor ($A$), User ($U$) and Director ($D$), and 6 relations: $M \rightleftharpoons A$, $M \rightleftharpoons U$ and $M \rightleftharpoons D$. We select 14 (task-irrelevant features such as id and url are ignored) numerical and categorical features from the original features for movie objects. Movie ($M$) objects are labeled by 4 classes: comedy, documentary, drama, and horror.

\textbullet\
\textbf{ACM.}
The dataset is provided by the authors of HAN \cite{han}. It is downloaded from ACM digital library\footnote{\url{https://dl.acm.org/}} in 2010, including data from 14 representative computer science conferences. We construct an HIN with 3 object types: Paper ($P$), Author ($A$) and Subject ($S$), and 4 relations: $P \rightleftharpoons A$ and $P \rightleftharpoons S$. Paper ($P$) objects are labeled by 3 research areas: data mining, database and computer network, and their features are the TF-IDF representations of their titles.

\textbullet\
\textbf{DBLP.}
The dataset is provided by the authors of HAN \cite{han}, which is extracted from 4 research areas of DBLP bibliography\footnote{\url{https://dblp.org/}}. The 4 research areas are: data mining (DM), database (DB), artificial intelligence (AI) and information retrieval (IR)\footnote{DM: ICDM, KDD, PAKDD, PKDD, SDM; DB: SIGMOD, VLDB, PODS, EDBT, ICDE; AI: AAAI, CVPR, ECML, ICML, IJCAI; IR: ECIR, SIGIR, WWW, WSDM, CIKM.}. Based on the dataset, we construct an HIN with 4 types: Paper ($P$), Author ($A$), Conference ($C$) and Term ($T$), and 6 relations: $P \rightleftharpoons A$, $P \rightleftharpoons C$ and $P \rightleftharpoons T$. Author ($A$) objects are labeled with the four research areas according to the conferences where they published papers \cite{han}.
Although paper objects have titles as their features, the titles provide very similar information as the terms connected to papers. Hence, we do not incorporate them as real features for papers, so that we can answer an important research question: whether ie-HGCN can well exploit useful structural features conveyed by meta-paths to accomplish the task without informative object features.

\subsection{Baselines}
\label{sec:methods}
We compare ie-HGCN against three GCN methods for homogeneous graphs: GraphSAGE (GSAGE for short), GCN and GAT; one HIN embedding method: metapath2vec (MP2V for short); five GCN methods for HINs: HAN, HAHE, DHNE (HIN embedding component of ActiveHNE), HetSANN (HetSA for short) and GTN; one ie-HGCN variant: ie-HGCN$_{mean}$. Details are as follows.

\textbullet\
\textbf{GraphSAGE (GSAGE)} \cite{graphsage}:
It is a homogeneous method that learns a function to aggregate features from a node's neighborhood. We use the convolutional mean-based aggregator, which corresponds to a rough, linear approximation of localized spectral convolution.

\textbullet\
\textbf{GCN} \cite{gcn}:
It is the state-of-the-art graph convolutional method for homogeneous graphs.

\textbullet\
\textbf{GAT} \cite{gat}:
It is designed for homogeneous graphs. For each node, it aggregates neighbor representations via the importance scores learned by node-level attention.

\textbullet\
\textbf{metapath2vec (MP2V)} \cite{metapath2vec}:
It is the state-of-the-art HIN embedding method. It first performs random walks guided by user-specified meta-paths and then uses the heterogeneous skip-gram to learn object representations. It cannot learn the importance of these input meta-paths.

\textbullet\
\textbf{HAN} \cite{han}:
It transforms an HIN into several homogeneous graphs via given symmetric meta-paths and uses GAT to perform object-level aggregation. Then, by attention mechanism, it fuses object representations learned from different meta-path based graphs.

\textbullet\
\textbf{HAHE} \cite{hahe}:
It is similar to HAN, except that it initializes the features of the target objects as the meta-path based structural features. Thus, it cannot exploit object features.

\textbullet\
\textbf{DHNE}:
ActiveHNE \cite{activehne} is an active learning method. For a
fair comparison, we use its Discriminative Heterogeneous Network Embedding (DHNE) component. It only considers fixed-length meta-paths, and cannot learn the importance of meta-paths.

\textbullet\
\textbf{HetSANN (HetSA)} \cite{hetsann}:
It is a heterogeneous method which directly uses attention mechanism to aggregate heterogeneous neighbors. We use the variant HetSANN.M.R.V which achieves the best performance as reported. The attention is implemented by sparse operations.

\textbullet\
\textbf{GTN} \cite{gtn}:
It is a heterogeneous method which considers all possible by computing all possible meta-path based graphs, and then performs graph convolution on the resulting graphs.

\textbullet\
\textbf{ie-HGCN$_{mean}$}:
It is a variant of ie-HGCN. We replace the type-level attention with the element-wise mean function. We use this method to show the effectiveness of the type-level attention.

\begin{table*}
  \centering
  \tabcolsep=0.17cm
  \caption{Object Classification Results.}
    \renewcommand{\arraystretch}{1.3}
    \begin{tabular}{|c|c|c|c|c|c|c|c|c|c|c|c|c|c|}
    \hline
    Dataset & Metrics (\%) & Training & GSAGE & GCN & GAT & MP2V & HAHE & HAN & DHNE & HetSA & GTN & ie-HGCN$_{mean}$ & ie-HGCN \\ \hline
    \multirow{8}[0]{*}{DBLP} & \multirow{4}[0]{*}{Micro F1} 
          &    20\% & 88.82 & 91.55 & 90.97 & 90.15 & 93.57 & 92.24 & 84.45 & 93.36 & 93.41 & 93.68 & \textbf{94.26} \\
          &  & 40\% & 88.81 & 91.10 & 91.20 & 90.81 & 93.61 & 92.40 & 84.61 & 93.72 & 93.84 & 93.55 & \textbf{94.22} \\
          &  & 60\% & 88.68 & 90.48 & 90.80 & 89.82 & 93.65 & 92.80 & 86.77 & 93.85 & 94.01 & 94.48 & \textbf{95.54} \\
          &  & 80\% & 88.87 & 91.72 & 91.73 & 90.89 & 94.38 & 93.08 & 87.36 & 94.03 & 94.46 & 95.20 & \textbf{96.48} \\
          \cline{2-14}
          & \multirow{4}[0]{*}{Macro F1} 
          &    20\% & 87.87 & 90.60 & 91.96 & 90.43 & 93.11 & 93.11 & 83.99 & 92.28 & 92.82 & 93.21 & \textbf{93.85} \\
          &  & 40\% & 87.98 & 90.17 & 92.16 & 89.73 & 93.78 & 93.30 & 84.80 & 92.51 & 93.34 & 93.05 & \textbf{93.83} \\
          &  & 60\% & 88.05 & 89.73 & 91.84 & 90.48 & 93.45 & 93.70 & 86.24 & 93.17 & 93.53 & 94.00 & \textbf{95.25} \\
          &  & 80\% & 88.29 & 90.99 & 92.55 & 90.97 & 94.24 & 93.99 & 86.82 & 93.48 & 93.77 & 94.72 & \textbf{96.29} \\
    \hline
    \multirow{8}[0]{*}{ACM} & \multirow{4}[0]{*}{Micro F1} 
          &    20\% & 81.47 & 78.80 & 74.18 & 66.74 & 77.17 & 73.58 & 76.21 & 78.57 & 77.85 & 78.73 & \textbf{81.93} \\
          &  & 40\% & 80.86 & 78.64 & 72.01 & 69.01 & 78.19 & 77.44 & 78.41 & 78.62 & 78.84 & 80.23 & \textbf{82.10} \\
          &  & 60\% & 80.31 & 77.78 & 76.18 & 71.68 & 78.09 & 76.47 & 78.97 & 79.25 & 79.27 & 83.26 & \textbf{83.73} \\
          &  & 80\% & 81.12 & 79.75 & 77.20 & 73.27 & 80.86 & 76.13 & 79.02 & 81.13 & 79.64 & 83.96 & \textbf{84.22} \\
          \cline{2-14}
          & \multirow{4}[0]{*}{Macro F1} 
          &    20\% & 63.40 & 60.19 & 58.18 & 50.92 & 53.87 & 64.69 & 64.94 & 57.89 & 51.34 & 68.95 & \textbf{69.79} \\
          &  & 40\% & 62.35 & 59.12 & 61.14 & 51.91 & 54.82 & 64.39 & 65.67 & 58.80 & 53.65 & 69.17 & \textbf{69.31} \\
          &  & 60\% & 60.38 & 58.80 & 53.79 & 51.87 & 54.65 & 65.31 & 65.99 & 59.20 & 55.36 & 69.43 & \textbf{70.25} \\
          &  & 80\% & 59.60 & 60.25 & 59.36 & 54.81 & 58.84 & 66.26 & 67.61 & 61.08 & 56.28 & 69.36 & \textbf{69.42} \\
    \hline
    \multirow{8}[0]{*}{IMDB} & \multirow{4}[0]{*}{Micro F1} 
             & 20\% & 58.20 & 59.58 & 55.30 & 49.87 & 54.89 & 56.47 & 61.69 & 61.07 &   -   & 62.12 & \textbf{64.94} \\
          &  & 40\% & 57.11 & 58.49 & 55.42 & 50.14 & 55.00 & 56.03 & 63.03 & 61.34 &   -   & 66.61 & \textbf{66.70} \\
          &  & 60\% & 59.89 & 59.81 & 55.14 & 50.83 & 55.01 & 57.00 & 64.16 & 62.21 &   -   & 66.44 & \textbf{68.22} \\
          &  & 80\% & 58.27 & 58.73 & 54.06 & 50.90 & 54.64 & 57.18 & 64.85 & 63.25 &   -   & 69.04 & \textbf{69.71} \\
          \cline{2-14}
          & \multirow{4}[0]{*}{Macro F1} 
          &    20\% & 40.93 & 29.26 & 30.07 & 20.95 & 21.47 & 45.87 & 51.27 & 50.41 &   -   & 54.19 & \textbf{56.60} \\
          &  & 40\% & 42.68 & 29.15 & 31.26 & 21.36 & 23.52 & 44.44 & 56.22 & 54.52 &   -   & 57.61 & \textbf{59.81} \\
          &  & 60\% & 39.64 & 30.95 & 29.82 & 20.89 & 21.72 & 44.82 & 55.95 & 54.37 &   -   & 55.47 & \textbf{60.84} \\
          &  & 80\% & 39.22 & 30.36 & 28.93 & 22.11 & 25.75 & 44.77 & 54.92 & 53.06 &   -   & 52.45 & \textbf{58.35} \\
    \hline
    \end{tabular}
  \label{tab:classify} 
\end{table*}

\subsection{Hyper-parameter Settings}
\label{sec:settings}
On each dataset, we randomly select $x$\% objects as training set, and the rest $(1-x)$\% are divided equally as validation set and test set, where $x \in \{20, 40, 60, 80\}$. For all the methods, we use exactly the same training/validation/test sets for fairness.
For all the methods, we search their respective hyper-parameters on the validation set. Specifically, we only investigate their proper hyper-parameter settings on the validation set of DBLP and use the same settings for ACM and IMDB. This can reflect whether the hyper-parameter setting is sensitive w.r.t. datasets. The hyper-parameter settings of all the methods are detailed as follows.

\textbullet\
\textbf{Ours}:
To make model tuning easy, we set the same hidden representation dimensionality for all the object types in a layer. Specifically, we set the number of layers to 5. The first layer is the input layer, and its dimensionalities for different objects are determined by object features. For the other 4 hidden layers, the dimensionalities are all set to [64, 32, 16, 8]. The nonlinearity $\sigma$ is set to ELU function \cite{elu_nonlinear}. The hidden layer dimensionality ${d_a}$ of the type-level attention is set to 64. For optimization, we use Adam optimizer with a learning rate of 0.01, and the parameters are initialized by Xavier uniform distribution \cite{xavier}. We apply dropout to the output of each layer except the output layer, with dropout rate 0.5. The $l_2$ regularization weight is set to 5e-4. For a fair comparison, our ie-HGCN and ie-HGCN$_{mean}$ use the same hyper-parameter setting.

\textbullet\
\textbf{Baselines}:
Since GSAGE, GCN, GAT, MP2V, HAN and HAHE need user-specified meta-paths, we use the meta-paths used in the papers \cite{han,hahe}. Concretely, on DBLP, we use $APA$, $APTPA$ and $APCPA$. On ACM, we use $PAP$ and $PSP$. On IDMB, we use $MAM$, $MUM$ and $MDM$. For GSAGE, GCN, GAT and MP2V, we test them on homogeneous graphs constructed by the above meta-paths and report their best results. For all the baselines, we use the validation set of DBLP to tune hyper-parameters starting from their default settings. Their key hyper-parameters are set as follows. For GSAGE, the neighborhood sample size is set to 5. For GAT, HAN and HetSA, the number of attention head is set to 8. For HAHE, the batch size is set to 512, and the sample size of neighbors is set to 100. For DHNE, its number of layers is set to 3. For GTN, the number of channels set to 2. Its number of layers is set to 3. For MP2V, the window size and the negative sample size are set to 5, and the walk length is set to 100.

\subsection{Object Classification}
\label{sec:classify}
We conduct object classification to compare the performance of all the methods. Each method is randomly run 10 times, and the average Micro F1 and Macro F1 are reported in Table \ref{tab:classify}. Note that due to the high space complexity of GTN, it cannot make use of GPU on our datasets due to out of 12GB GPU memory. Therefore in this experiment, we run GTN on CPUs with 128GB main memory, as suggested by the authors \cite{gtn}. Even then, it runs out of 128GB main memory on IMDB. We can see, ie-HGCN achieves the best overall performance, and ie-HGCN$_{mean}$ outperforms the other baselines in most cases, which indicates the effectiveness of our proposed heterogeneous graph convolution for object-level aggregation. On the other hand, ie-HGCN performs better than ie-HGCN$_{mean}$, which shows the effectiveness that our proposed type-level attention can discover and exploit the most useful meta-paths for this task.

On DBLP, heterogeneous methods HAN and HAHE significantly outperform homogeneous methods GSAGE, GCN and GAT, while on ACM and IMDB, the former does not have much superiority than the latter. This may be because on DBLP, the heterogeneous structural features conveyed by meta-paths are more helpful for this task (see Section \ref{sec:attnstudy}), while on ACM and IMDB, the real features of the target objects are more helpful. Our ie-HGCN always achieves the best results on all the datasets, which indicates that it can not only exploit useful structural features but also take advantage of useful object features.
On DBLP, we can see DHNE performs much worse than HAN, HAHE, HetSA and GTN, because it only exploits fixed-length meta-paths and cannot learn their importance. Even homogeneous methods GSAGE, GCN and GAT perform better than DHNE, as they can exploit useful meta-paths that previous researchers have empirically chosen. 
On DBLP, HetSA and GTN perform better than DHNE, which may be because they can exploit all possible meta-paths. GTN performs better than HetSA, which may be because that GTN can correctly discover and exploit useful meta-paths for this task while HetSA cannot. However, GTN performs worse than ie-HGCN, the reason of which should be that it is not flexible enough to capture the complexity of different objects. MP2V performs worst in most cases, which indicates the superiority of graph convolutional methods over traditional network embedding methods.

\subsection{Interpretability Study}
\label{sec:attnstudy}

\begin{table}
\caption{Useful Meta-paths Discovered by ie-HGCN on DBLP.}
\centering
\renewcommand{\arraystretch}{1.3}

\subtable[Mean Attention Coefficients.]{    
\begin{tabular}{|c|c|c|c|}
\hline
Layers & $\Omega$ & $\mathcal{N}_\Omega$ & Coefficients \\
\hline
\multirow{4}[0]{*}{1-2}  & $P$     & [$P^{Self}$, $A$, $C$, $T$] & [0.06, 0.06, 0.82, 0.06] \\
    \cline{2-4}          & $A$     & [$A^{Self}$, $P$] & [0.50, 0.50] \\
    \cline{2-4}          & $C$     & [$C^{Self}$, $P$] & [0.63, 0.37] \\
    \cline{2-4}          & $T$     & [$T^{Self}$, $P$] & [0.50, 0.50] \\
\hline
\multirow{4}[0]{*}{2-3}  & $P$     & [$P^{Self}$, $A$, $C$, $T$] & [0.64, 0.04, 0.27, 0.05] \\
    \cline{2-4}          & $A$     & [$A^{Self}$, $P$] & [0.20, 0.80] \\
    \cline{2-4}          & $C$     & [$C^{Self}$, $P$] & [0.37, 0.63] \\
    \cline{2-4}          & $T$     & [$T^{Self}$, $P$] & [0.06, 0.94] \\
\hline
\multirow{4}[0]{*}{3-4}  & $P$     & [$P^{Self}$, $A$, $C$, $T$] & [0.25, 0.25, 0.25, 0.25] \\
    \cline{2-4}          & $A$     & [$A^{Self}$, $P$] & [0.49, 0.51] \\
    \cline{2-4}          & $C$     & [$C^{Self}$, $P$] & [0.42, 0.58] \\
    \cline{2-4}          & $T$     & [$T^{Self}$, $P$] & [0.19, 0.81] \\
\hline
4-5                      & $A$     & [$A^{Self}$, $P$] & [0.43, 0.57] \\
\hline
\end{tabular}
\label{tab:mean-attn-coeff}
}

\subtable[Importance Scores.]{ 
\begin{tabular}{|c|c|c|}
\hline
Meta-paths & Merged Paths & Importance Scores \\ \hline
\multirow{6}[0]{*}{$CPA$} & $CP-P-PA$ & 0.82 * 0.64 * 0.25 * 0.57 = 0.0748 \\
\cline{2-3} & $C-CP-PA$ & 0.63 * 0.27 * 0.25 * 0.57 = 0.0242 \\
\cline{2-3} & $C-C-CPA$ & 0.63 * 0.37 * 0.25 * 0.57 = 0.0332 \\
\cline{2-3} & $C-CPA-A$ & 0.63 * 0.27 * 0.51 * 0.43 = 0.0373 \\
\cline{2-3} & $CP-PA-A$ & 0.82 * 0.64 * 0.51 * 0.43 = 0.1151 \\
\cline{2-3} & $CPA-A-A$ & 0.82 * 0.80 * 0.49 * 0.43 = 0.1382 \\ \hline
$CPTPA$ & $CPTPA$   & 0.82 * 0.94 * 0.25 * 0.57 = 0.1098 \\ \hline
$CPAPA$ & $CPAPA$   & 0.82 * 0.80 * 0.25 * 0.57 = 0.0935 \\ \hline
$CPCPA$ & $CPCPA$   & 0.82 * 0.63 * 0.25 * 0.57 = 0.0736 \\ \hline
\end{tabular}
\label{tab:path-scores}
}

\label{tab:dblp-metapath}
\end{table}

One salient feature of ie-HGCN is the ability to evaluate all possible meta-paths, discover and exploit the most useful ones for a specific task. We provide experimental evidences in this subsection. 
Firstly, we show the most useful meta-paths in a global sense, which are discovered by ie-HGCN for the task of classifying author objects in DBLP. Specifically, we first compute the mean attention distribution in each block, which is shown in Table \ref{tab:mean-attn-coeff}. Then we compute the importance scores of meta-paths based on these mean distributions. Table \ref{tab:path-scores} shows the computation details. Here, note that we need to merge equivalent paths that stand for the same meta-path. For object type $\Omega$, we use $\Omega - \Omega$ to denote the dummy self-relation, and use $\Omega\Omega$ to denote the real self-relation. Thus, the top-6 paths in Table \ref{tab:path-scores} are all equivalent to $CPA$. The importance score of $CPA$ is the sum of the scores of the top-6 paths, i.e., 0.4228.
It indicates that $CPA$ is the most useful meta-path for the task of author object classification. This is reasonable. The semantic meaning of $CPA$ is ``the conferences where authors have published papers''. This correctly reflects the fact that in the DBLP dataset, the class of an author (i.e. his/her research area) is labeled according to the conferences where he/she have published papers \cite{han}. Besides, $CPAPA$ and $CPTPA$ are also useful for the task. $CPAPA$ indicates that in addition to the conferences where an author himself/herself has published papers, the conferences where his/her coauthors have published papers are also useful. $CPTPA$ suggests we should further consider conferences where the published papers share a lot of common terms with those written by the author. The last meta-path $CPCPA$ is also intuitive. Notice that a paper can only be published in one conference. The meta-path $CPC$ essentially does not introduce information of other conferences to a conference. Hence, we can interpret $CPCPA$ as $CPA$.

Regarding the best meta-path for each object, we also find the results are intuitive. For example, ie-HGCN correctly classifies Yoshua Bengio as ``AI'' (see Section \ref{sec:data} for details of labels), and assigns the highest score to $CPCPA$ for him. This is intuitive since all the 7 papers connected to him in our dataset are from ``AI'' conferences such as ICML. On the other hand, ie-HGCN correctly classifies the scholar Chen Chen as ``DM'', and assigns the highest score to $CPAPA$. This is also reasonable. In our dataset, he has published 3 papers in ``DB'' conferences and 2 papers in ``DM'' conferences, but all the 5 papers are co-authored with Jiawei Han, who has published many papers in ``DM'' conferences such as KDD. These observations indicate that ie-HGCN is able to evaluate the importance of meta-paths according to the information of different objects. 

Regarding baselines, HAN and HAHE assign the largest attention coefficient to meta-path $APCPA$ \cite{han,hahe}. GSAGE, GCN, GAT and MP2V also achieve the best classification results when their input meta-path is $APCPA$. This meta-path means that we should resort to authors who have published papers in the same conferences as the target author. However, this is less effective since it indirectly exploits the conference information. These methods cannot directly exploit the most useful meta-path $CPA$, because they can only perform homogeneous graph convolution which requires constructing homogeneous graphs by symmetric meta-paths.

Our ie-HGCN discovers useful meta-paths of $PSP$ and $PAP$ on ACM, and $MUM$, $MAM$ and $UM$ on IMDB, for their respective object classification tasks. Most of them are widely used in previous works. It also indicates the effectiveness of real features of target objects, since the target objects are typically connected to their input features.

\begin{figure}
  \centering
  \subfigure[Scalability]{\includegraphics[width=0.49\columnwidth]{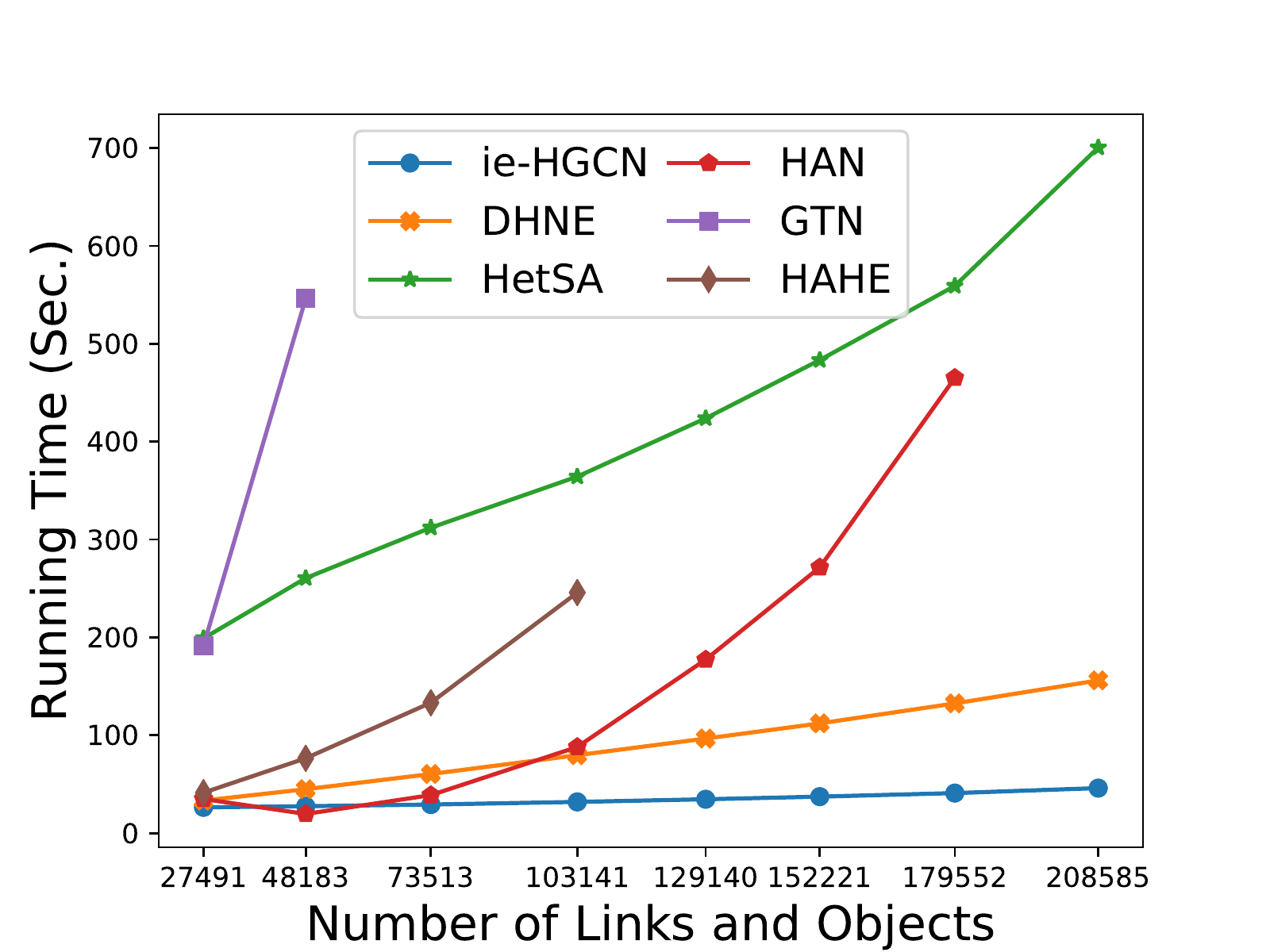}\label{fig:scalability}}
  \subfigure[Depth]{\includegraphics[width=0.49\columnwidth]{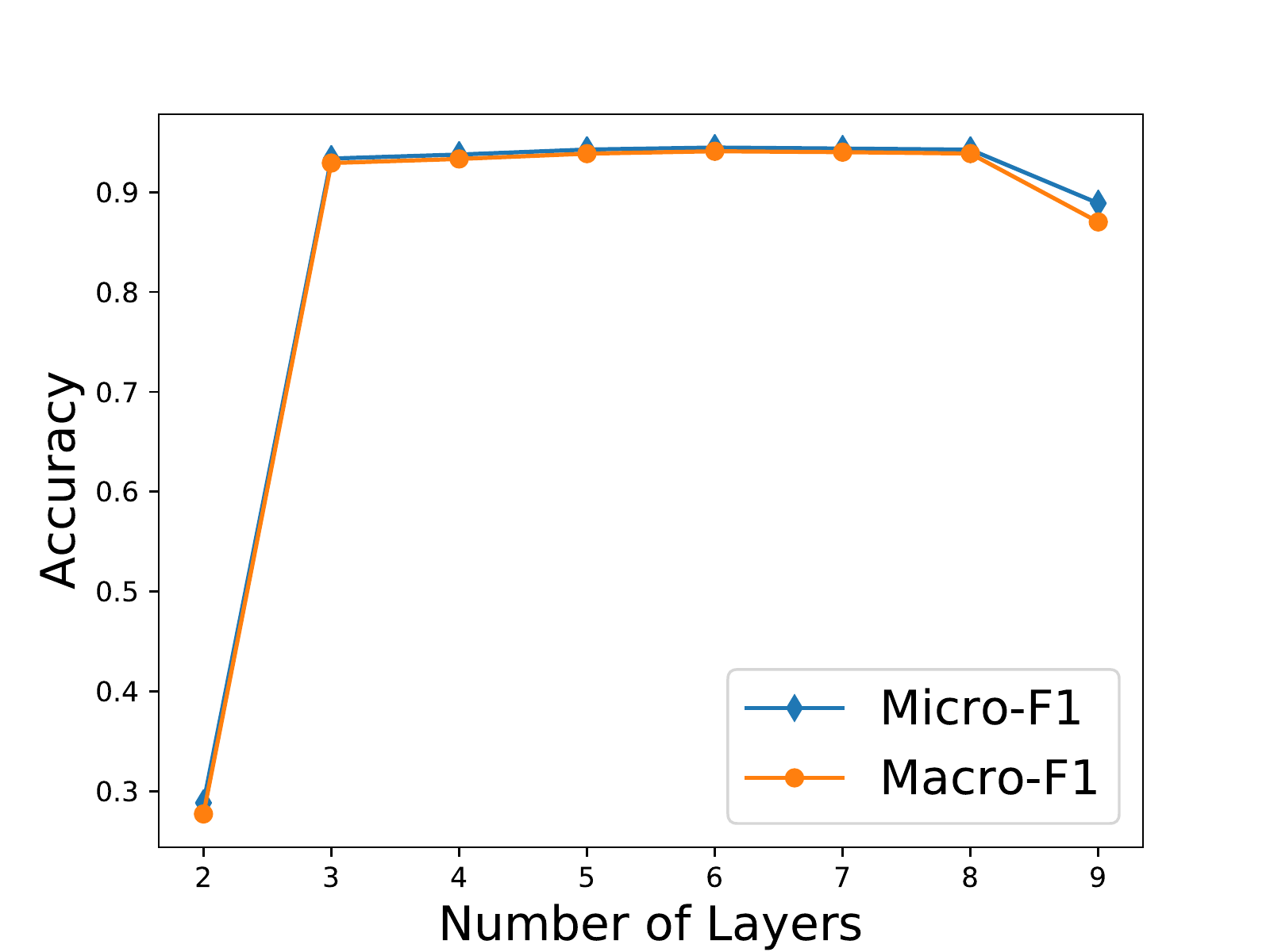}\label{fig:depth}}
  \caption{
  (a): Running time (seconds) of all the heterogeneous GCN methods w.r.t. the total number of links and objects;
  (b): Classification performance of ie-HGCN w.r.t. its number of layers.
  }
  \label{fig:scala-depth}
\end{figure}



\begin{figure*}
  \centering
  \subfigure[IMDB]{\includegraphics[width=0.67\columnwidth]{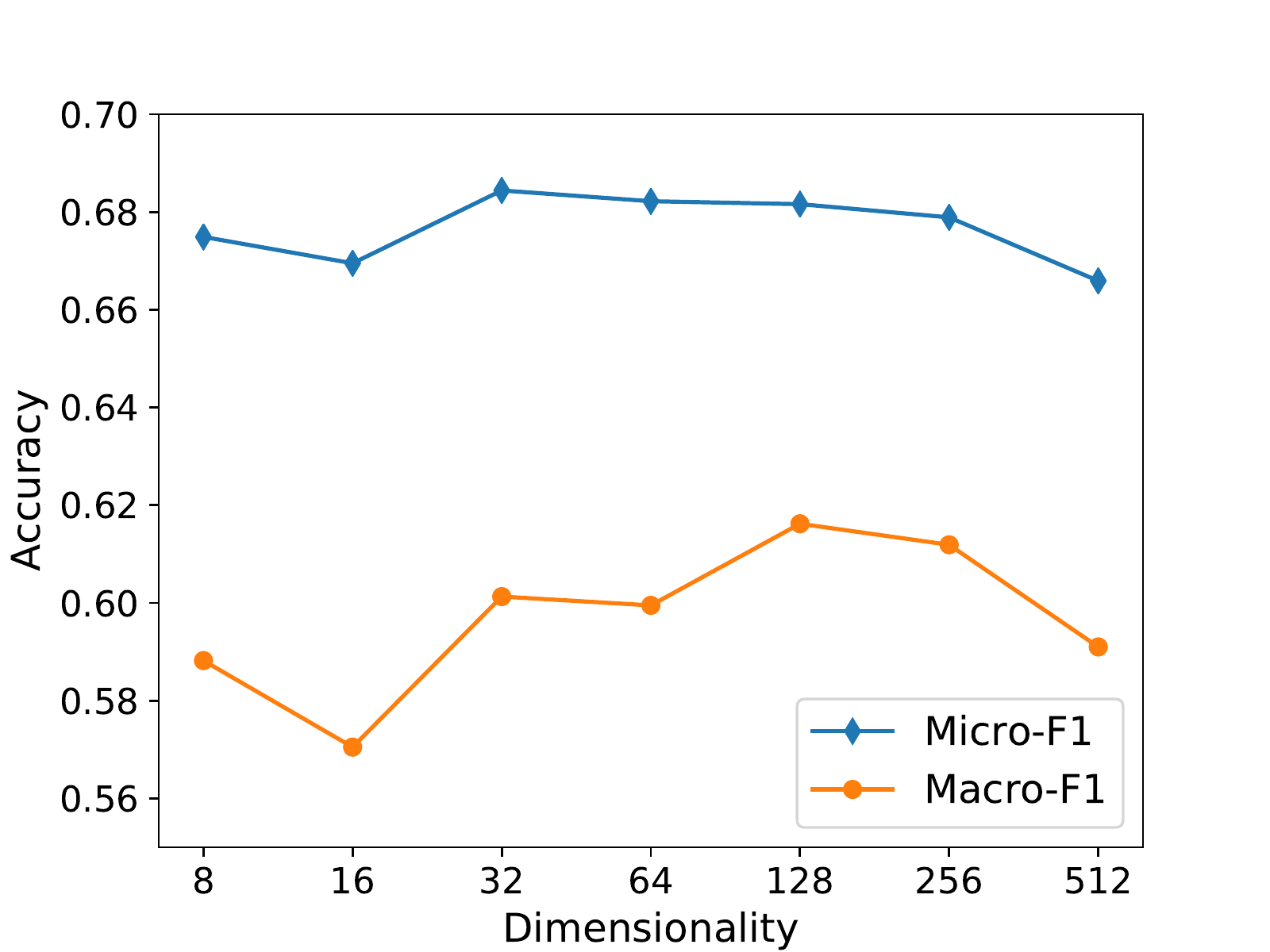}\label{fig:param-imdb}}
  \subfigure[ACM]{\includegraphics[width=0.67\columnwidth]{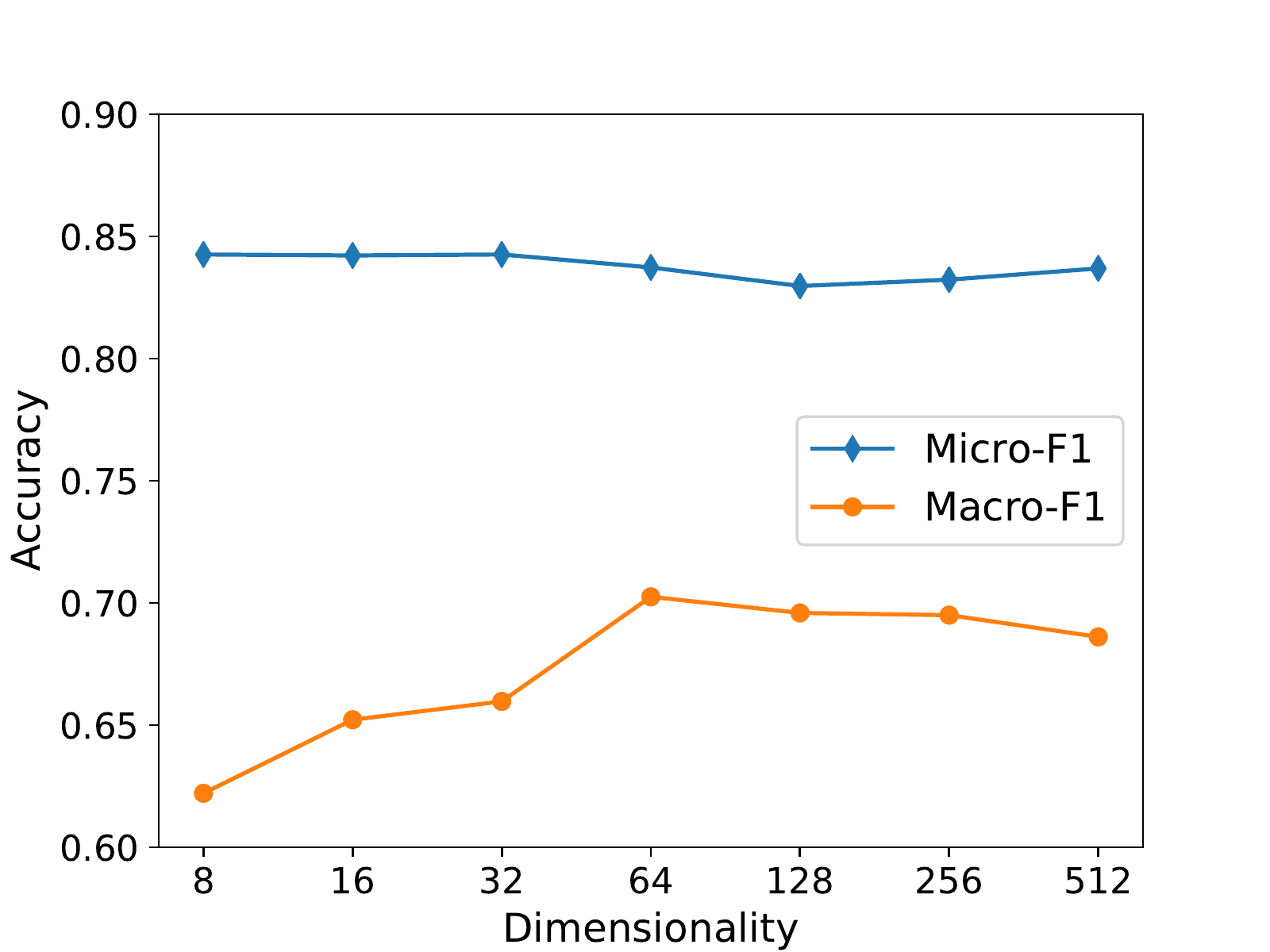}\label{fig:param-acm}}
  \subfigure[DBLP]{\includegraphics[width=0.67\columnwidth]{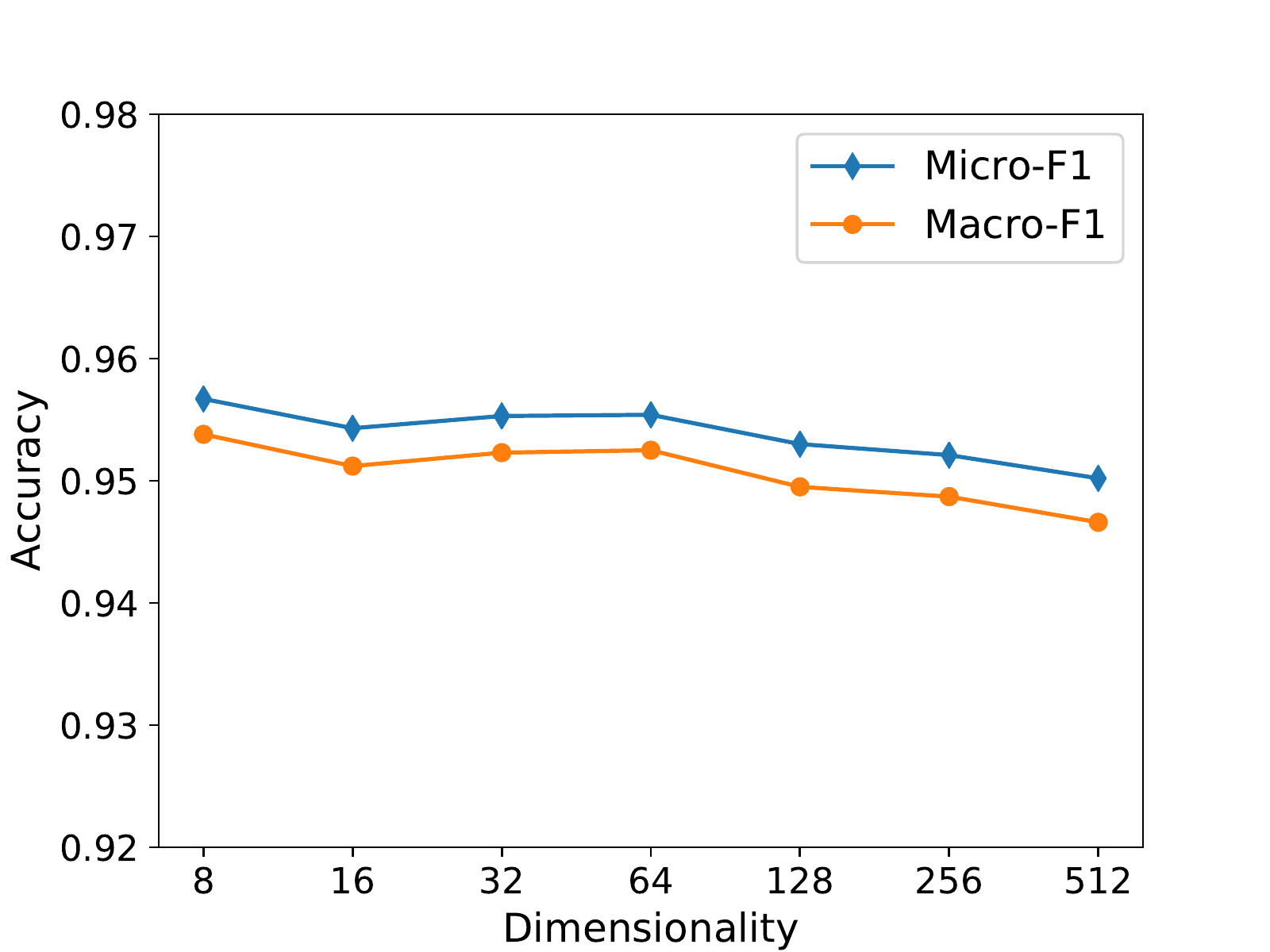}\label{fig:param-dblp}}
  \caption{Classification performance of ie-HGCN w.r.t. the hidden layer dimensionality ${d_a}$ of the type-level attention.}
  \label{fig:param}
\end{figure*}

\subsection{Scalability}
\label{sec:scala}
Based on the original DBLP dataset provided by HAN \cite{han}, we use different numbers of author objects to induce 8 HINs with different scales. Let us denote the scale by a tuple: (author numbers, total objects, total links), i.e. ($|{\mathcal V}^{A}|$, $|{\mathcal V}|$, $|{\mathcal E}|$). Thus, from small to large, the scales of resulting HINs are: (800, 6183, 21308), (1500, 9799, 38384), (2500, 13935, 59578), (4000, 18785, 84356), (5500, 22969, 106171), (7000, 26327, 125894), (10000, 31775, 147777), (14475, 37791, 170794).

We test and compare the scalability of the heterogeneous GCN methods on these constructed HINs with different scales. All the methods are randomly run 10 times on GPU. The average running time (seconds) w.r.t. the total number of links+objects in the HINs is reported in Figure \ref{fig:scalability}. We can see, ie-HGCN achieves the best scalability. The time cost of ie-HGCN increases linearly w.r.t. the HIN scale, which helps to verify the time complexity analyzed in Section \ref{sec:analys-time-complexity}. 
HAN, HAHE and HetSA perform worse than DHNE and ie-HGCN because the former need to perform object-level attention, which is computationally inefficient in practice. The time cost of HAN and HAHE increases more dramatically than that of HetSA because the former need to construct meta-path based graphs which are very dense, and the latter implements the object-level attention by sparse operations. HAHE performs worse than HAN, since it uses the high dimensional meta-path based structural features as input features. DHNE is less efficient than our ie-HGCN, which should be because that its hidden representation concatenation requires significantly more parameters to obtain the output representations of each layer than our type-level attention. GTN shows the worst scalability due to its square time complexity. HAHE, HAN and GTN cannot run on large-scale HINs due to out of 12 GB GPU memory.

\subsection{Depth Study}
\label{sec:depth}
For ie-HGCN, the more layers, the more complex and rich semantics can be captured. We implement 8 instances of ie-HGCN with layers increasing from 2 to 9. Specifically, fixing other hyper-parameters, we set the layer (except for input layer) dimensionalities of these model instances as: [64], [64, 32], [64, 32, 16], [64, 32, 16, 8], [64, 64, 32, 16, 8], [64, 64, 64, 32, 16, 8], [64, 64, 64, 64, 32, 16, 8], [64, 64, 64, 64, 64, 32, 16, 8].

We test their classification performance on DBLP. Each of these model instances is randomly run 10 times, and the average Micro F1 and Macro F1 scores are reported in Figure \ref{fig:depth}. We can see, when the model has 2 layers, the performance is very poor. This is not surprising, since the ie-HGCN model with 2 layers can only consider meta-paths with length less than 2. As discussed in Section \ref{sec:attnstudy}, in order to accurately classify author objects on DBLP, it is critical to capture and fuse the information from related conference objects. However, there is no 1-hop meta-path between authors and conferences in the network schema of DBLP (Figure~\ref{fig:toy-dblp}). When the depth becomes 3, the performance is promoted dramatically, since it is possible for ie-HGCN to exploit meta-path $CPA$. Then, the performance grows slightly as the depth increases until it achieves its best when the depth is 6. After that, the performance starts to decrease possibly due to overfitting.

\subsection{Hyper-parameter Study}
\label{sec:param}
In this subsection, we investigate the sensitivity of ie-HGCN's performance to the hidden layer dimensionality ${d_a}$ of the type-level attention. With other hyper-parameters fixed, we gradually increase ${d_a}$ from 8 to 512 and report Micro F1 and Macro F1 in Figure \ref{fig:param}. We can see, on DBLP, the performance is not very sensitive to ${d_a}$. On IMDB and ACM, Micro F1 is not very sensitive, while Macro F1 is more sensitive. Considering that Macro F1 is sensitive to skewed classes in classification, this can be explained by the fact that the classes in IMDB and ACM are skewed, while those in DBLP are balanced. The general pattern is: in the beginning, the performance grows as the dimensionality gradually increases; then the performance begins to decline, which should be because of overfitting with more parameters in the attention module. The overall inflection point is at the dimensionality of 64. Thus, we set ${d_a}=64$ for ie-HGCN.

\section{Conclusion}
\label{sec:conclu}
In this paper, we propose ie-HGCN to learn representations of objects in an HIN. To address the heterogeneity, we first project the representations of different types of neighbor objects into a common semantic space. Then we use row-normalized adjacency matrices to perform the object-level aggregation. We formally show that the first two steps intrinsically perform heterogeneous spectral graph convolution on HINs. Finally, we use the proposed type-level attention to aggregate the convolved representations of different types of neighbor objects. Our ie-HGCN automatically evaluates all possible meta-paths in an HIN, discovers and exploits the most useful meta-paths for a specific task, which brings good interpretability of the model. The theoretical analysis and the scalability experiment show that it is efficient. Extensive experiments show ie-HGCN outperforms several state-of-the-art methods.

\section*{Acknowledgments}
The authors would like to thank the authors of HAN \cite{han} and HAHE \cite{hahe} for their source codes and raw datasets.
This research was supported by the National Natural Science Foundation of China (Grant Nos. 61672409, 61936006, 61876144, 61876145), the Key Research and Development Program of Shaanxi (Program No.2020ZDLGY04-07), Shaanxi Province Science Fund for Distinguished Young Scholars (Grant No. 2018JC-016) and the Fundamental Research Funds for the Central Universities (Grant Nos. JB190301, JB190305).

\ifCLASSOPTIONcaptionsoff
  \newpage
\fi

\bibliographystyle{IEEEtran}
\bibliography{hgcn}

\begin{IEEEbiography}[{\includegraphics[width=0.9in, height=1.25in,clip,keepaspectratio]{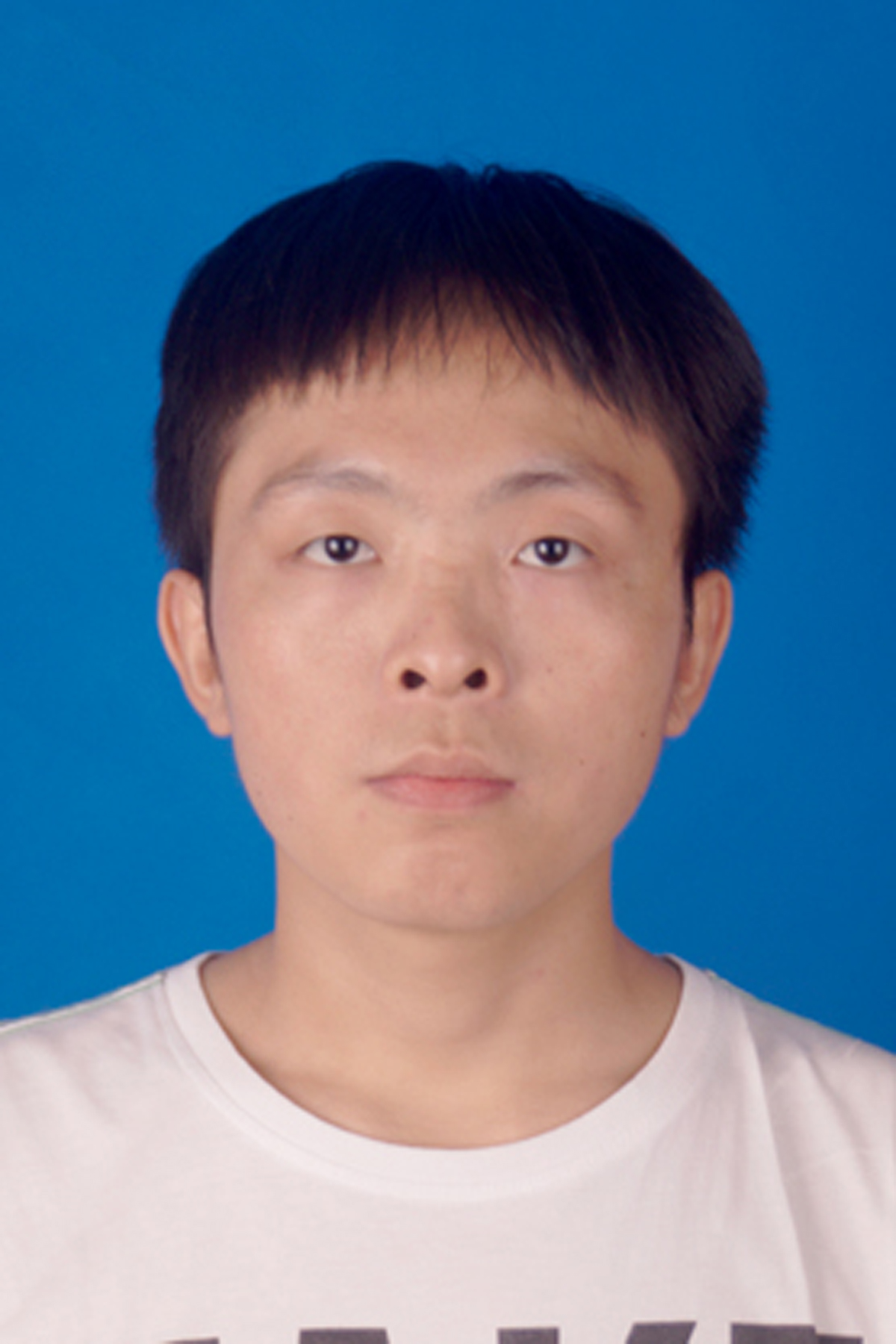}}]{Yaming Yang} received the B.S. degree in Computer Science and Technology from Xidian University, China, in 2015. He is currently a second year Ph.D. student in the School of Computer Science and Technology, Xidian University. His research interests include data mining and machine learning.
\end{IEEEbiography}

\begin{IEEEbiography}[{\includegraphics[width=0.9in, height=1.25in,clip,keepaspectratio]{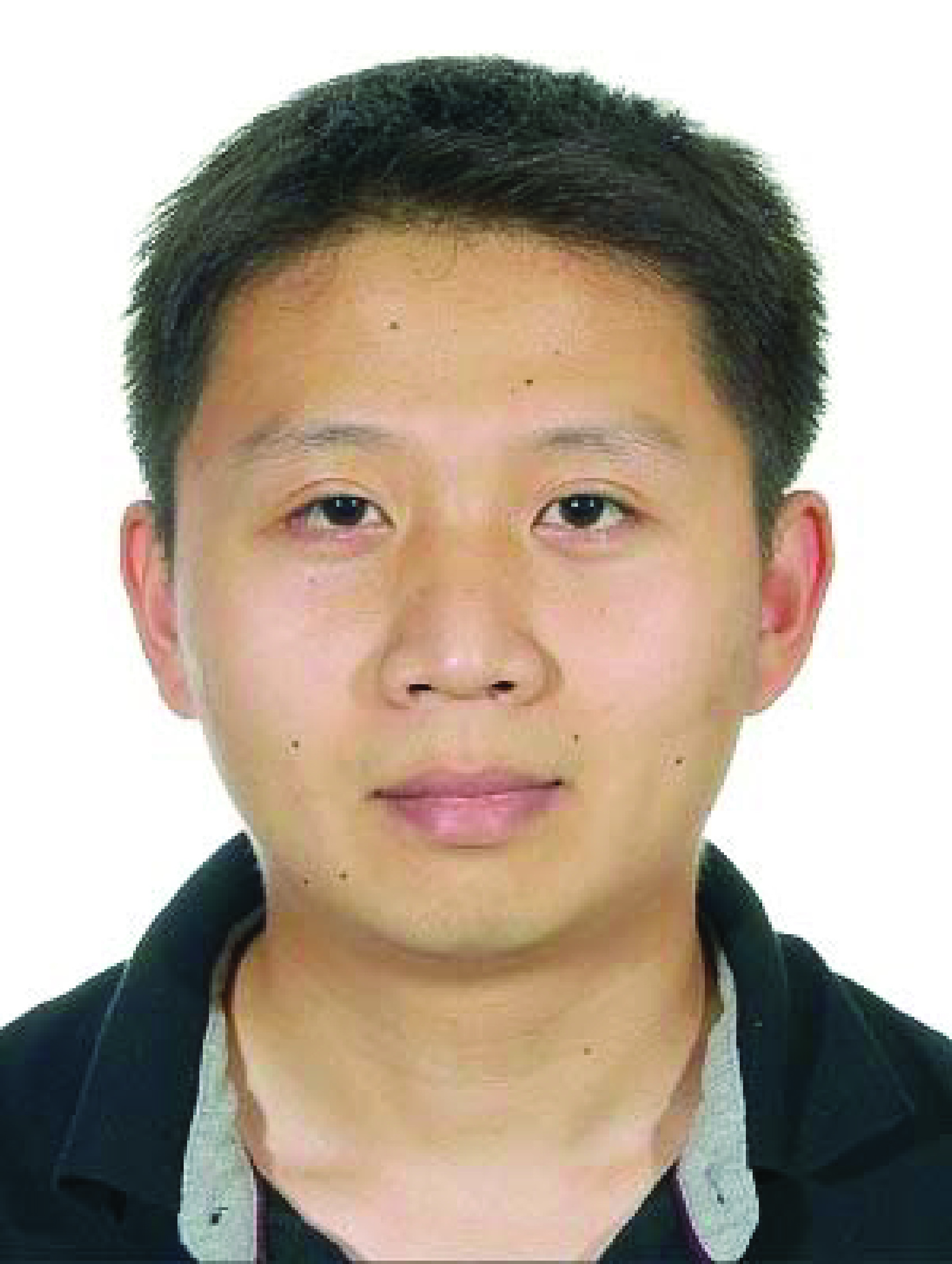}}]{Ziyu Guan} received the B.S. and Ph.D. degrees in Computer Science from Zhejiang University, Hangzhou China, in 2004 and 2010, respectively. He had worked as a research scientist in the University of California at Santa Barbara from 2010 to 2012, and as a professor in the School of Information and Technology of Northwest University, China from 2012 to 2018. He is currently a professor with the School of Computer Science and Technology, Xidian University. His research interests include attributed graph mining and search, machine learning, expertise modeling and retrieval, and recommender systems.
\end{IEEEbiography}

\begin{IEEEbiography}[{\includegraphics[width=0.9in, height=1.25in,clip,keepaspectratio]{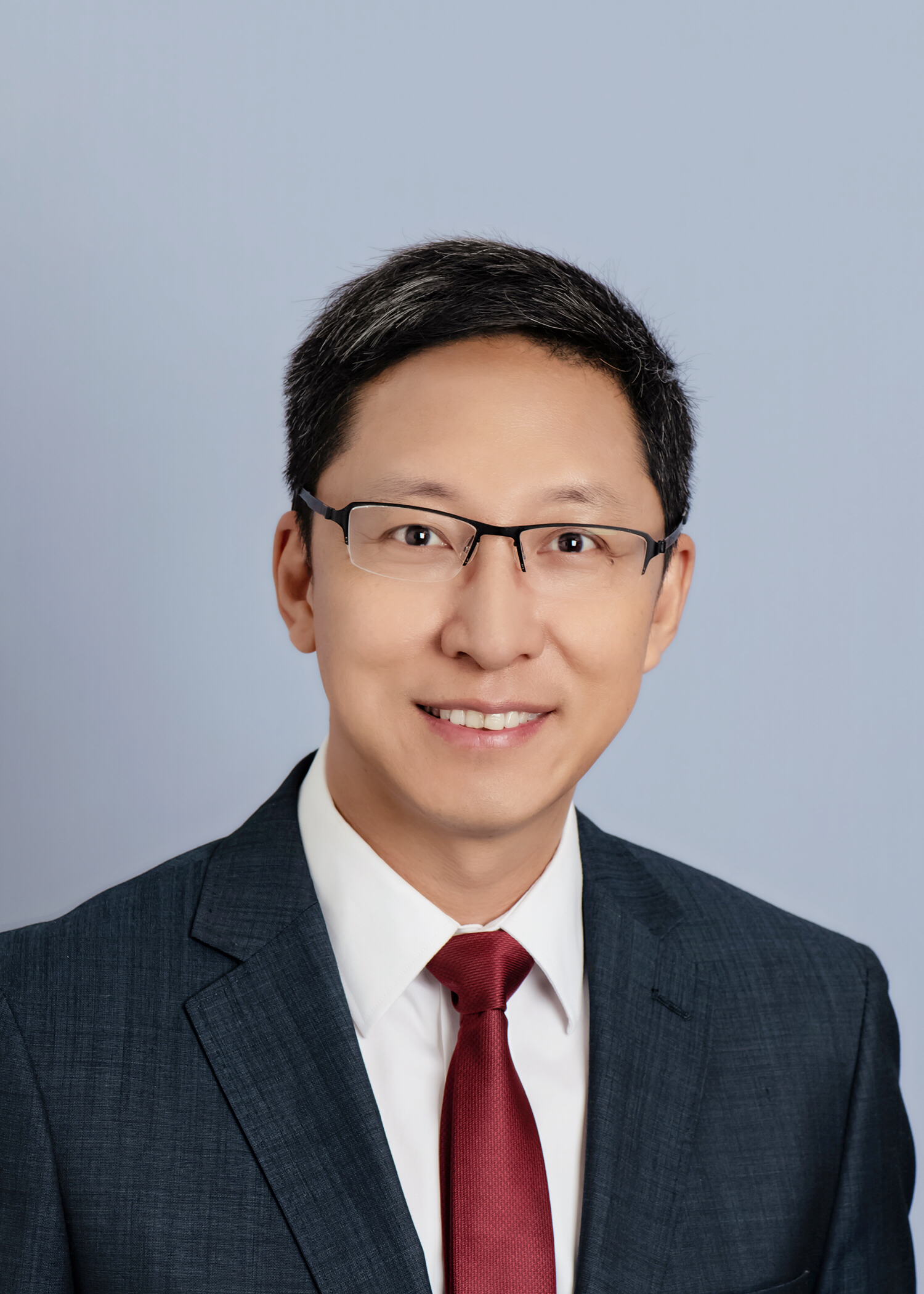}}]{Jianxin Li} received the PhD degree in computer science from the Swinburne University of Technology, Australia, in 2009. He is an Associate Professor in Data Science and the director of the Smart Networks Lab at the School of Information Technology, Deakin University. His research interests include graph database query processing \& optimization, social network analytics \& computing, complex network data representation learning traffic, and personalized online learning analytics.
\end{IEEEbiography}

\begin{IEEEbiography}[{\includegraphics[width=0.9in, height=1.25in,clip,keepaspectratio]{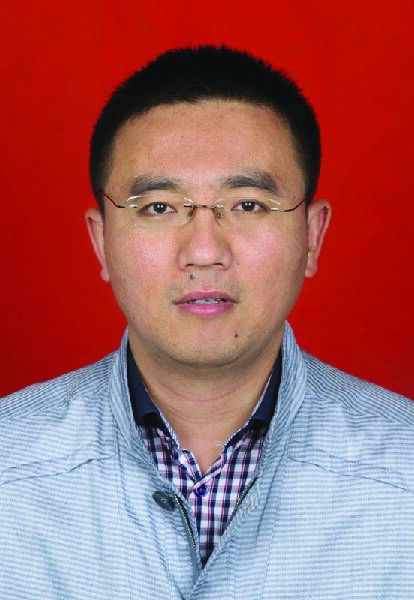}}]{Wei Zhao} received the B.S., M.S. and Ph.D. degrees from Xidian University, Xi’an, China, in 2002, 2005 and 2015, respectively. He is currently a professor in the School of Computer Science and Technology at Xidian University. His research direction is pattern recognition and intelligent systems, with specific interests in attributed graph mining and search, machine learning, signal processing and precision guiding technology.
\end{IEEEbiography}

\begin{IEEEbiography}[{\includegraphics[width=0.9in, height=1.25in,clip,keepaspectratio]{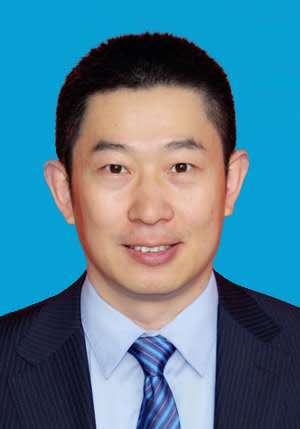}}]{Jiangtao Cui} received the M.S. and Ph.D. degree both in Computer Science from Xidian University, Xian, China in 2001 and 2005 respectively. During 2007 and 2008, he has been with the Data and Knowledge Engineering group working on high-dimensional indexing for large scale image retrieval, in the University of Queensland, Australia. He is currently a professor in the School of Computer Science and Technology, Xidian University. His current research interests include data and knowledge engineering, data security, and high-dimensional indexing. 
\end{IEEEbiography}

\begin{IEEEbiography}[{\includegraphics[width=0.9in, height=1.25in,clip,keepaspectratio]{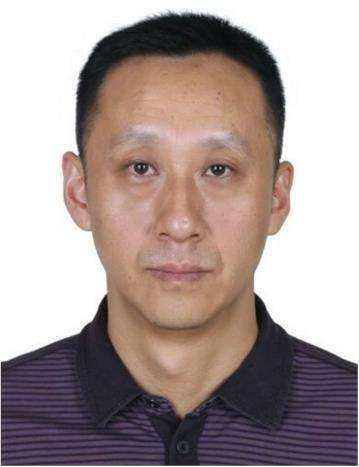}}]{Quan Wang} received his B.S., M.S. and Ph.D. degrees in Computer Science from Xidian University, Xi'an, China, in 1992, 1997 and 2008, respectively. He is now a professor at the School of Computer Science and Technology and the vice president of Xidian University. His research interests include hardware security, embedded system, wireless networks, and 3-D printing. He is a member of IEEE, a distinguished member of China Computer Federation (CCF) and Councilor of ACM Xi'an Section.
\end{IEEEbiography}

\end{document}